\def\year{2022}\relax
\documentclass[letterpaper]{article} 

\usepackage{aaai22}  
\usepackage{times}  
\usepackage{helvet}  
\usepackage{courier}  
\usepackage[hyphens]{url}  
\usepackage{graphicx} 

\usepackage{algorithm}
\usepackage[noend]{algpseudocode}
\usepackage[inline]{enumitem}

\usepackage{newfloat}
\usepackage{listings}
\floatstyle{ruled}
\newfloat{listing}{tb}{lst}{}
\floatname{listing}{Listing}

\usepackage{booktabs} 
\usepackage{amsmath}
\usepackage{amssymb}
\usepackage{dsfont}
\usepackage{xspace} 
\usepackage{amsthm}  
\usepackage{thmtools}
\usepackage{thm-restate}
\usepackage{etoolbox}
\usepackage{bm}
\usepackage{amsfonts}
\usepackage[capitalise, noabbrev]{cleveref}
\usepackage{pifont}
\usepackage{mathtools}
\usepackage{mathabx}
\usepackage{mathtools}
\usepackage{subcaption}
\usepackage{mathrsfs}
\usepackage[percent]{overpic}
\usepackage{cleveref}

\usepackage{mymacros}

\usepackage{natbib}  
\usepackage{caption} 
\DeclareCaptionStyle{ruled}{labelfont=normalfont,labelsep=colon,strut=off} 
\definecolor{redp}{rgb}{0.78, 0.03, 0.08}
\definecolor{greenp}{rgb}{0.0, 0.51, 0.5}
\definecolor{yellowp}{rgb}{0.59, 0.44, 0.09}
\definecolor{greencol}{rgb}{0.0,0.4,0.0}
\definecolor{fcolor}{rgb}{0.8, 0.4, 0.0}
\definecolor{bluep}{rgb}{205,219,194}

\definecolor{vibrantBlue}{RGB}{0, 119, 187}
\definecolor{vibrantCyan}{RGB}{51, 187, 238}
\definecolor{vibrantTeal}{RGB}{0, 153, 136}
\definecolor{vibrantOrange}{RGB}{238, 119, 51}
\definecolor{vibrantRed}{RGB}{204, 51, 17}
\definecolor{vibrantMagenta}{RGB}{238, 51, 119}
\definecolor{vibrantGrey}{RGB}{100, 100, 100}

\definecolor{fuqqzz}{rgb}{0.9568627450980393,0,0.6}
\definecolor{ccqqqq}{rgb}{0.8,0,0}
\definecolor{xdxdff}{rgb}{0.49019607843137253,0.49019607843137253,1}
\definecolor{qqqqff}{rgb}{0,0,1}
\definecolor{crimason}{rgb}{0.8627,0.0784,0.2352}
\definecolor{green}{rgb}{0,0.5019,0}
\definecolor{ududff}{rgb}{0.30196078431372547,0.30196078431372547,1}
\definecolor{darkblue}{rgb}{0,0,139}

\usepackage{pgfplots}
\pgfplotsset{compat=1.15}
\usepackage{mathrsfs}

\usetikzlibrary{arrows,fit,positioning,shapes}  
\usetikzlibrary{backgrounds}
\usetikzlibrary{decorations.pathreplacing}

\AtBeginEnvironment{proof}{\small}
\allowdisplaybreaks[4]
\usetikzlibrary{bayesnet}

\title{
Lifelong Hyper-Policy Optimization with Multiple Importance Sampling Regularization
}
\author{
Pierre Liotet\textsuperscript{\rm 1}, Francesco Vidaich\textsuperscript{\rm 2 }, Alberto Maria Metelli\textsuperscript{\rm 1}, Marcello Restelli\textsuperscript{\rm 1}
}
\affiliations{
\textsuperscript{\rm 1}Politecnico di Milano, Milan, Italy\\
\textsuperscript{\rm 2}University of Padova, Padua, Italy\\
pierre.liotet@polimi.it
}

\begin{document}

\maketitle


\begin{abstract}
Learning in a lifelong setting, where the dynamics continually evolve, is a hard challenge for current reinforcement learning algorithms. Yet this would be a much needed feature for practical applications. 
In this paper, we propose an approach which learns a hyper-policy, whose input is time, that outputs the parameters of the policy to be queried at that time. 
This hyper-policy is trained to maximize the estimated future performance, efficiently reusing past data by means of importance sampling, at the cost of introducing a controlled bias. We combine the future performance estimate with the past performance to mitigate catastrophic forgetting.
To avoid overfitting the collected data, we derive a differentiable variance bound that we embed as a penalization term. Finally, we empirically validate our approach, in comparison with state-of-the-art algorithms, on realistic environments, including water resource management and trading.
\end{abstract}

\section{Introduction}
In the most common setting, Reinforcement Learning~\citep[RL,][]{sutton2018reinforcement} considers the interaction between an agent and an environment in a sequence of episodes. The agent progressively adapts its policy, but the dynamics of the environment, typically, remain unchanged. Most importantly, the agent can experience multiple times the same portion of the environment. 
However, this usual setting is sometimes not met in real applications. Hence several modifications have been proposed to model different, more realistic, scenarios. 
One of them is \emph{non-stationary} RL~\citep{bowerman1974nonstationary}, which considers that the episodes can follow different distributions, or even that the distribution changes within each episode. 
The change can either be \emph{abrupt}, when a clear separation between tasks evolving through time is present, or \emph{smooth}, when the environment evolution displays some regularity \wrt time. Non-stationarity can arise from diverse causes and can be interpreted as a form of partial knowledge on environment \cite{khetarpal2020towards}. 
Learning in non-stationary environments has been diffusely addressed in the literature~\cite{garcia2000solving, ghate2013linear, lesner2015nonstationary}. 
Nevertheless, in these works, the agent-environment interaction based on episodes is preserved, so that the same region of non-stationary behavior can be experienced multiple times by the agent.

Still moving towards a more realistic setting, another modification is the \emph{lifelong} interaction with the environment~\citep{silver2013lifelong, brunskill2014pac}. Here, the separation in episodes vanishes and, therefore, there is no clear distinction between learning and testing. 
Moreover, given the never-ending nature of this interaction, the agent is not allowed to reset the environment and, consequently, it might not be possible to visit twice some portions of the environment. 
Thus, the agent aims at exploiting the experience collected in the past to optimize its future performance. 
In this sense, Lifelong Learning (LL) can be considered closer to the intuitive idea of learning for human agents. More technically, LL requires the agent to readily adapt its behavior to the environment evolution, as well as keeping memory of past behaviors in order to leverage this knowledge on future similar phases~\citep{khetarpal2020towards}. This represents, indeed, a critical trade-off, peculiar of the lifelong setting. Indeed, if the agent displays a highly non-stationary behavior, the samples collected in the past would be poorly informative and, consequently, hardly usable to estimate the future performance. Instead, preferring a more stationary behavior would favor sample reuse, at the cost of sacrificing the optimality of the learned behavior.


In this paper, we consider the RL problem with a lifelong interaction between an agent and the environment, where the environment's dynamics \emph{smoothly} evolve over time.  
We address this problem by designing a hyper-policy, responsible for selecting the best policy to be played at time $t$. This way, we decouple the problem of learning in a non-stationary setting, by assigning to the hyper-policy level the management of the dependence on time and to the policy level the action to be played given a state (Section~\ref{sec:formulation}).
This hyper-policy is trained with an objective composed of
the future performance, the ultimate quantity to be maximized, and the past performance. Although the past performance is not the direct interest of our agent, it is included to constrain the hyper-policy to perform well on past samples, thus mitigating \textit{catastrophic forgetting}. 
Future performance is estimated through multiple importance sampling.
To avoid overfitting, we additionally penalize the hyper-policy for the variance of the estimations. Rather than estimating this quantity, that would inject further uncertainty, we derive a differentiable upper-bound allowing a gradient based optimization. This penalization, involving a divergence between past and future hyper-policies, has the indirect effect of quantifying and controlling the \quotes{amount} of non-stationarity selected by the agent (Section~\ref{sec:objective}). 
We propose a practical policy-gradient optimization of the objective, which we name POLIS, for Policy Optimization in Lifelong learning through Importance Sampling.
After having revised the literature (Section~\ref{sec:relatedWorks}), we provide an experimental evaluation on realistic domains, including a trading environment and a water resource management problem, in comparison with state-of-the-art baselines (Section~\ref{sec:experiments}). The proofs of the results presented in the main paper are reported in Appendix~\ref{apx:proofs}.

\section{Preliminaries}
In this section, we report the necessary background that will be employed in the following sections.

\textbf{Lifelong RL}~~A Non-Stationary Markov Decision Process~\citep[MDP,][]{puterman2014markov} is defined as $\mathcal{M} = (\Xs, \As, P, R, \gamma, D_0)$, where $\Xs$ and $\As$ are the state and action spaces respectively, $P= (P_t)_{t \in \Nat}$ is the transition model that for every decision epoch $t \in \Nat$ and $(x,a) \in \Xs \times \As$ assigns a probability distribution over the next state $x' \sim P_t(\cdot|x,a)$, $R = (R_t)_{t \in \Nat}$ is the reward distribution assigning for every $t \in \Nat$ and $(x,a) \in \Xs \times \As$ the reward $r \sim R_t(\cdot|x,a)$ such that $\|r\|_{\infty} \le R_{\max} < \infty $, $\gamma \in [0,1]$ is the discount factor, and $D_0$ is the initial state distribution. A non-stationary policy $\pi = (\pi_t)_{t \in \Nat}$ assigns for every decision epoch $t \in \Nat$ and state $x \in \Xs$ a probability distribution over the actions $a_t \sim \pi_t(\cdot|x)$. Let $T \in \Nat$ be the current decision epoch, let $\beta \in \Nat[1]$, we define the \emph{$\beta$-step ahead expected return} as:
\begin{align}\label{eq:Jbeta}
	J_{T,\beta}(\pi) =\sum_{t = T+1}^{T+\beta} \widehat{\gamma}^t \Epit[\pi] [r_{t}],
\end{align}
where $\widehat{\gamma}^t = \gamma^{t-T-1}$ and we denote with $\Epit$ the expectation under the visitation distribution induced by policy $\pi$ in MDP $\mathcal{M}$ after $t$ decision epochs.
A policy $\pi^{\star}_{T,\beta}$ is $\beta$-step ahead optimal if $\pi^{\star}_{T,\beta} \in \argmax_{\pi \in \Pi^\beta} J_{T,\beta}(\pi)$, where $\Pi^\beta$ is the set of non-stationary policies operating over $\beta$ decision epochs. 
In \emph{classical RL}, the agent's goal consists in maximizing $J_{0,H}$, where $H$ is the (possibly infinite) horizon of the task, having the possibility to collect \emph{multiple} episodes (not necessarily of length $H$). Instead, from the \emph{lifelong RL} perspective, the agent is interested in maximizing the $\infty$-step ahead expected return $J_{T,\infty}(\pi)$, having observed \emph{in the past only one} episode of length $T$, \ie optimizing for the future.

\textbf{Multiple Importance Sampling}~~Importance Sampling~\citep[IS,][]{mcbook} allows estimating the expectation $\mu = \E_{x \sim P}[f(x)]$ of a function $f$ under a \emph{target} distribution $P$ having samples collected with a sequence of \emph{behavioral} distributions $(Q_j)_{j\in[\![1,J]\!]}$ such that $P \ll Q_j$, \ie $P$ is absolutely continuous \wrt $Q_j$, for all $j \in [\![1,J]\!]$. Let $p$ and $(q_j)_{j\in[\![1,J]\!]}$ be the density functions corresponding to $P$ and $(Q_j)_{j\in[\![1,J]\!]}$, then, the resulting unbiased estimator is:
\begin{align*}
	\widehat{\mu} = \sum_{j=1}^J \frac{1}{N_j}\sum_{i=1}^{N_j} \beta_j(x_{ij})\frac{p(x_{ij})}{q_j(x_{ij})}f(x_{ij}),
\end{align*}
where $\{x_{ij}\}_{i=1}^{N_j} \sim Q_j$ and $(\beta_j(x))_{j\in[\![1,J]\!]}$
are a partition of the unit for every $x \in \Xs$.
A common choice for the latter is to use the \emph{balance heuristic}~\citep[BH,][]{veach1995monte}, yielding $\beta_j(x)=\frac{N_jq_j(x)}{\sum_{k=1}^J N_k q_k(x)}$. 
Using BH, samples can be regarded as obtained from the mixture of the $(Q_j)_{j\in[\![1,J]\!]}$ distributions as $\Phi = \sum_{k=1}^J\frac{N_k}{N}Q_k$, with $N=\sum_{j=1}^{J} N_j$. 


\textbf{\Renyi divergence}~~
Let $\alpha \in [0,\infty]$, the $\alpha$-\Renyi divergence between two probability distributions $P$ and $Q$ such that $P \ll Q$ is defined as:
\begin{align*}
	D_{\alpha}(P\Vert Q) = \frac{1}{\alpha-1}\log \int_{\Xs} p(x)^\alpha q(x)^{1-\alpha} \de x.
\end{align*}
We denote with $d_{\alpha}(P\Vert Q) = \exp\{D_{\alpha}(P\Vert Q)\}$ the exponential $\alpha$-\Renyi divergence, linked to the $\alpha$-moment of the importance weight, \ie $\mathbb{E}_{x\sim Q}\left[\left(\frac{p(x)}{q(x)}\right)^{\alpha}\right] = d_{\alpha}(P\Vert Q)^{\alpha-1}$.

%
%
%
%
%
\section{Lifelong Parameter-Based Policy Optimization}\label{sec:formulation}
In this paper, we consider the Policy Optimization~\citep[PO,][]{deisenroth2013survey} setting in which the policy belongs to a set of parametric policies $\Pi_{\Theta} = \{\pi_{\vtheta} : \vtheta \in \Theta \subseteq \Reals^{d_1}\}$. In particular, we focus on the \emph{parameter-based} PO\footnote{We follow the taxonomy of~\cite{metelli2018policy}.} in which the policy parameter $\vtheta$ is sampled from a \emph{hyper-policy} $\nu_{\vrho}$ belonging, in turn, to a parametric set $\mathcal{N}_{\mathcal{P}} = \{\nu_{\vrho} : \vrho \in \mathcal{P} \subseteq \Reals^{d_2}\}$~\citep{sehnke2008parameter}. As opposed to \emph{action-based} PO in which policies $\pi_{\vtheta}$ needs to be stochastic for exploration reasons, in parameter-based PO we move the stochasticity to the hyper-policy $\nu_{\vrho}$ level and $\pi_{\vtheta}$ can be deterministic.

Optimizing the $\beta$-step ahead expected return in \cref{eq:Jbeta} requires, in general, considering non-stationary policies. From the PO perspective, this requirement can be fulfilled in two ways. The traditional way consists in augmenting the state $x$ with time the $t$ and, consequently, considering a policy of the form $\pi_{\vtheta}(\cdot|(x,t))$. This approach highlights the direct dependence of the action $a_t \sim \pi_{\vtheta}(\cdot|(x_t,t))$ on the time $t$. However, in several cases, it is convenient to track the evolution of the policy parameters $\vtheta$ as a function of the time $t$, whose dependence might be simpler compared to that of the action. In this latter approach, the one we adopt in this work, the policy parameter is sampled from a \emph{time-dependent} hyper-policy $\vtheta_t \sim \nu_{\vrho}(\cdot|t)$ and the policy depends on the state only $\pi_{\vtheta_t}(\cdot|x_t)$. We will refer to this setting as \emph{lifelong parameter-based PO}.
Refer to Figure~\ref{fig:graph_models} for a comparison of the graphical models of the two approaches.
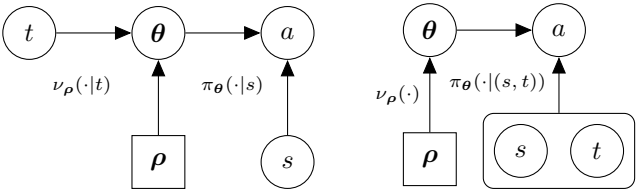
\begin{figure}
    \centering
    \begin{tikzpicture}
    \node[latent] (x) {$t$}; %
    \node[latent, right =of x] (z) {$\bm{\theta}$} ; %
    \node[latent, right=of z] (y) {$a$};
    \node[latent, rectangle, below=of z] (r) {$\bm{\rho}$};
    \node[latent, below=of y] (s) {$s$};
    \path[->] (x) edge node[above ] {} (z);
    \path[->] (z) edge node[above ] {} (y);
    \path[->] (r) edge node[above ] {} (z);
    \path[->] (s) edge node[above ] {} (y);
    \node (x) at (.7,-.7) [draw=none]  {\scriptsize $\nu_{\bm{\rho}}(\cdot|t)$};
    \node (x) at (2.7,-.7) [draw=none]  {\scriptsize $\pi_{\bm{\theta}}(\cdot|s)$};
\end{tikzpicture}
\hfill
\begin{tikzpicture}
    \node[latent] (z) {$\bm{\theta}$} ; %
    \node[latent, right=of z] (y) {$a$};
    \node[latent, rectangle, below=of z] (r) {$\bm{\rho}$};
    \node[latent, below left=1.1cm and .0cm of y] (s) {$s$};
    \node[latent, below right=1.1cm and .0cm of y] (t) {$t$};
	\node[draw, rectangle, rounded corners, wrap=(s) (t), rectangle, inner sep=4pt] (w) {};    
    \path[->] (z) edge node[above ] {} (y);
    \path[->] (r) edge node[left ] {\scriptsize $\nu_{\bm{\rho}}(\cdot)$} (z);
    \path[->] (w) edge node[above ] {} (y);
    \node (x) at (.9,-.7) [draw=none]  {\scriptsize $\pi_{\bm{\theta}}(\cdot|(s,t))$};
\end{tikzpicture}
    \caption{Graphical models of the two approaches to model non-stationarity of the environment:  the non-stationarity is handled at the parameter selection level (left), the non-stationarity is handled at the action selection level (right).}
    \label{fig:graph_models}
\end{figure}



In this setting, we aim at learning a hyper-policy parameter $\vrho^\star_{T,\beta}$ maximizing the \textit{$\beta$-step ahead expected return}:
\begin{align}\label{eq:opt}
	\vrho^\star_{T,\beta} \in \argmax_{\vrho \in \mathcal{P}} J_{T,\beta}(\vrho) = \sum_{t=T+1}^{T+\beta} \widehat{\gamma}^t \Epit[\vrho][r_t],
\end{align} 
where $\Epit[\vrho][\cdot]$ is a shorthand for  $\E_{\vtheta \sim \nu_{\vrho}(\cdot|t)}[\Epit[\pi_{\vtheta}][\cdot]]$. 

\section{Lifelong Parameter-Based PO via Multiple Importance Sampling}\label{sec:objective}

In this section, we propose an estimator for $J_{T,\beta}(\vrho) $ (Section~\ref{sec:betaEst}), we analyze its bias (Section~\ref{sec:isBias}) and variance (Section~\ref{sec:isVariance}), and we propose a novel surrogate objective accounting for the estimation uncertainty (Section~\ref{subsec:algo_optim}). 

\subsection{$\beta$-Step Ahead Expected Return Estimation}\label{sec:betaEst}
The main challenge we face in estimating $J_{T,\beta}(\vrho)$ is that it requires evaluating hyper-policy $\nu_{\vrho}$ in the future, while having samples from the past only. Since the environment evolves smoothly, it is reasonable to use the past data to approximate the future dynamics and IS to correct the hyper-policy behavior mismatch from past to future. More specifically, in this section, we study how to leverage the history of the past $\alpha$ samples $\mathcal{H}_{T,\alpha} = (\vtheta_t,r_t)_{t \in [\![T-\alpha+1,T]\!]}$ in order to estimate the $\beta$-step ahead expected return $J_{T,\beta}(\vrho)$. 

As a preliminary step, we illustrate the estimation of the $s$-step ahead expected reward $\mathbb{E}_s^{\vrho}[r_s]$. For every $s \in [\![T+1,T+\beta]\!]$ we employ the following MIS estimator that makes use of the history $\mathcal{H}_{T,\alpha}$:
\begin{align}\label{eq:hStepEstimator}
    \widehat{r}_{s} = \sum_{t=T-\alpha+1}^{T} \omega^{T-t} \frac{\nu_{\vrho}(\vtheta_t|s)}{ \sum_{k=T-\alpha+1}^T \omega^{T-k} \nu_{\vrho}(\vtheta_t|k)} r_{t},
\end{align}
where $\omega \in [0,1]$ is an exponential weighting parameter.
The importance sampling correction $\frac{\nu_{\vrho}(\vtheta_t|s)}{ \sum_{k=T-\alpha+1}^T \omega^{T-k} \nu_{\vrho}(\vtheta_k|k)}$ addresses the mismatch between the hyper-policies in the future $\nu_{\vrho}(\cdot|s)$ and those in the past $\nu_{\vrho}(\cdot|k)$. 
The reader may have noticed that these importance weights are not using the exact BH weights. Indeed, we have adapted the heuristic to include our knowledge that the environment is \emph{smoothly changing}. 
With BH, each past sample would have been weighted equally whereas our sampling mixture probability, proportional to $ \sum_{k=T-\alpha+1}^T \omega^{T-k} \nu_{\vrho}(\vtheta_t|k)$, gives more weight to recent samples thanks to the parameter $\omega$ which exponentially discounts samples as they are collected far from current time $T$~\citep{jagerman2019people}. 

Using $\widehat{r}_{s}$ as building block, we now propose the estimator for the $\beta$-step ahead expected return $\widehat{J}_{T,\alpha,\beta}(\vrho)$ that is obtained as the discounted sum of the $s$-step ahead expected reward estimators of Equation~\eqref{eq:hStepEstimator}:
\begin{align*}
    \widehat{J}_{T,\alpha,\beta}({\vrho}) &= \sum_{s=T+1}^{T+\beta} \widehat{\gamma}^{s} \widehat{r}_{s} \\
    & = \sum_{t=T-\alpha+1}^{T} \omega^{T-t} \frac{\sum_{s=T+1}^{T+\beta} \widehat{\gamma}^s \nu_{\vrho}(\vtheta_t|s)}{\sum_{k=T-\alpha+1}^T \omega^{T-k} \nu_{\vrho}(\vtheta_t|k)} r_{t}.
\end{align*}

This estimator could be, in principle, employed in an optimization algorithm but, as common in IS-based estimators, we would incur in the following undesired effect. In order to increase $\widehat{J}_{T,\alpha,\beta}({\vrho})$, the agent can either increase the probability of good actions for the future policies $\nu_{\vrho}(\vtheta_t|s)$ (the numerator of the importance weight) or decrease the probability of the same good actions for the policies from the past $\nu_{\vrho}(\vtheta_t|k)$ (the denominator of the importance weight). 
The latter phenomenon, akin to catastrophic forgetting, is clearly undesired, but can be easily spotted by looking at the past reward. Specifically, we propose to adjust the objective function with the return of the last $\alpha$ steps, called \textit{$\alpha$-step behind expected return}:
\begin{align*}
    \widecheck{J}_{T,\alpha}({\vrho}) = \frac{1}{C_\omega} \sum_{t=T-\alpha+1}^{T} \omega^{T-t}  \widecheck{\gamma}^t r_{t},
\end{align*}
where $\widecheck{\gamma}^t = \widecheck{\gamma}^{t-T+\alpha-1}$, $C_\omega = \frac{1-\omega^\alpha}{1-\omega}$ if $\omega<1$ otherwise $C_\omega = \alpha$. 
Putting all together, we obtain the objective:
\begin{align*}
    \overline{J}_{T,\alpha,\beta}(\vrho) = \widehat{J}_{T,\alpha,\beta}({\vrho})  +  \widecheck{J}_{T,\alpha}({\vrho}).
\end{align*}

\subsection{Bias Analysis}\label{sec:isBias}
In this section, we analyze the bias of estimator $\widehat{J}_{T,\alpha,\beta}(\vrho)$, under suitable regularity conditions on the environment and on the hyper-policy model. In particular, we will require that the environment and the hyper-policy are \emph{smoothly changing}. We formalize the intuition in the following assumptions.

\begin{ass}[Smoothly Changing Environment]\label{ass:sce}
For every $t,t' \in \Nat$, and for every policy $\pi$ it holds for some Lipschitz constant $0 \le L_{\mathcal{M}} < \infty$:
\begin{align*}
    \left| \left(\mathbb{E}_{t}^{\pi}  -\mathbb{E}_{t'}^{\pi} \right)\left[ r \right] \right| \le L_{\mathcal{M}} \left|t - t'\right|.
\end{align*}
\end{ass}
\begin{ass}[Smoothly Changing Hyper-policy]\label{ass:sch}
For every $t,t' \in \Nat$, and for every time-dependent hyper-policy $\vrho \in \mathcal{P}$ it holds for some Lipschitz constant $0 \le L_{\nu} < \infty$:
\begin{align*}
    \left\| \nu_{\vrho}(\cdot|t) - \nu_{\vrho}(\cdot|t') \right\|_1 \le L_{\nu} \left|t - t'\right|.
\end{align*}
\end{ass}
Thus, Assumption~\ref{ass:sce} prescribes that executing the same policy $\pi$ in different time instants $t$ and $t'$ results in an expected reward that can be bounded proportionally to the time distance. A similar requirement is requested by Assumption~\ref{ass:sch}, involving the  total variation distance between time-dependent hyper-policies. Under these assumptions, we provide the following bias bound. We denote with $\mathbb{E}_{T,\alpha}^{\vrho}$ the expectation under the probability distribution induced by the joint hyper-policy $\prod_{t=T-\alpha+1}^{T} \nu_{\vrho}(\cdot\vert t)$ in the MDP.

\begin{restatable}{lemma}{biasBound}
\label{lem:bias_bound}
Under Assumptions~\ref{ass:sce} and~\ref{ass:sch}, the bias of the estimator $\widehat{J}_{T,\alpha,\beta}(\vrho)$, for $\omega<1$, can be bounded as:
\begin{align*}
    &\left|J_{T,\beta}(\vrho) - \mathbb{E}^{\vrho}_{T,\alpha}[\widehat{J}_{T,\alpha,\beta}] \right|\\  
    & \qquad \le (L_{\mathcal{M}} + 2R_{\max}L_\nu)  C_\gamma(\beta)\left(\frac{\omega}{1-\omega} +\frac{1}{1-\gamma}\right),
\end{align*}
{\thinmuskip=1mu
\medmuskip=1mu
\thickmuskip=1mu
where $C_\gamma(\xi)=\frac{1-\gamma^\xi}{1-\gamma}$ if $\gamma<1$ otherwise $C_\gamma(\xi)=\xi$ for $\xi \ge 1$.
}
\end{restatable}
A tighter, but more intricate bias bound,  and a derivation for the case $\omega=1$ can be found in Appendix~\ref{apx:proofs}.
Some observations are in order. First, we note the role of $\omega$ in controlling the bias: the smaller $\omega$, the smaller the bias. Second, the bound is proportional to the Lipschitz constants governing the non-starionarity of the environment and of the hyper-policy. It is worth noting that in a fully stationary setting (\ie $L_{\mathcal{M}}=L_{\nu}=0$), the estimator is unbiased.

\subsection{Variance Analysis}\label{sec:isVariance}
 Before showing the construction of the surrogate objective, we derive in this section a bound on the variance of $\overline{J}_{T,\alpha,\beta}(\vrho)$ that involves the \Renyi divergence. To this purpose, we denote with $\mathbb{V}\mathrm{ar}^{\vrho}_{T,\alpha}$ the variance under the probability distribution induced by the joint hyper-policy $\prod_{t=T-\alpha+1}^{T} \nu_{\vrho}(\cdot\vert t)$ in the MDP.
\begin{restatable}{lemma}{ppvarbound}
\label{pp:var_bound}
The variance of the objective $\overline{J}_{T,\alpha,\beta}$ can be bounded as:
{\thinmuskip=0mu
\medmuskip=0mu
\thickmuskip=0mu
\begin{align*}
    &\mathbb{V}\mathrm{ar}^{\vrho}_{T,\alpha}\left[\overline{J}_{T,\alpha,\beta}(\vrho)\right]\leq 
    2 R_{\max}^{2} \Bigg(C_\gamma(\alpha)^2 + C_\gamma(\beta)^2 \\
    &  \times d_2 \left( \frac{1}{C_\gamma(\beta)} \sum_{s=T+1}^{T+\beta} \widehat{\gamma}^s \nu_{\vrho}(\cdot|s)  \Bigg\|   \frac{1}{C_\omega} \sum_{t=T-\alpha+1}^{T} \omega^{T-t} \nu_{\vrho}(\cdot|t)  \Bigg) \right).
\end{align*}
}
\end{restatable}
The variance bound resembles the ones usually provided in the context of off-policy estimation and learning~\citep[\eg][]{metelli2018policy, papini2019optimistic, metelli2020importance}. The first addendum accounts for the variance of the estimator component $\widecheck{J}_{T,\alpha}({\vrho}) $ that does not involve importance sampling, whereas the second refers to $ \widehat{J}_{T,\alpha,\beta}({\vrho})$, based on importance sampling. Indeed, this latter term comprises the exponentiated $2$-\Renyi divergence between two mixture hyper-policies. Unfortunately, even in presence of convenient distributions, like Gaussians, the \Renyi divergence between mixtures does not admit a closed form~\cite{papini2019optimistic}. In Appendix~\ref{app:var_bound}, we discuss several approaches, based on variational upper-bounds, to provide a usable version of such a divergence. In the following, we report the upper-bound that we will use in practice.
\begin{restatable}{lemma}{pdivbound}
\label{pp:divergence_bound}
    The divergence between mixtures of Lemma~\ref{pp:var_bound} can be bounded as:
\begin{align*}
    d_2 & \left(\frac{1}{C_\gamma(\beta)} \sum_{s=T+1}^{T+\beta} \widehat{\gamma}^s \nu_{\vrho}(\cdot|s) \left\Vert \frac{1}{C_\omega} \sum_{t=T-\alpha+1}^{T} \omega^{T-t} \nu_{\vrho}(\cdot|t)  \right.\right) 
    \\
    & \;\; \le \frac{C_\omega}{C_\gamma(\beta)^2} \underbrace{\left(\sum_{s=T+1}^{T+\beta} \frac{\widetilde{\gamma}^{s}}{\left(\sum\limits_{ t=T-\alpha+1}^{ T}\frac{\omega^{T-t}}{d_{2}(\nu_{\vrho}(\cdot\vert s)\left\Vert \nu_{\vrho}(\cdot\vert t)\right)}\right)^{\frac{1}{2}}}\right)^2}_{{B}_{T,\alpha,\beta}(\vrho)}.
\end{align*}
\end{restatable}

\subsection{Surrogate Objective}\label{subsec:algo_optim}
The direct optimization of the objective $\overline{J}_{T,\alpha,\beta}(\vrho)$ makes the hyper-policy overfit the non-stationary process on the last $\alpha$ steps. To allow for a better generalization on future unseen variations, following the idea of~\citet{metelli2018policy}, we regularize the objective with the bound on the variance of Lemma~\ref{sec:isVariance}. The following concentration bound, based on Cantelli's inequality, is the theoretical grounding of our surrogate objective.
\begin{restatable}{thr}{thrlbound}\label{thr:thrlbound}
\label{th:rl_bound}
    For every $\delta\in (0,1)$, with probability at least $1-\delta$, it holds that: 
\begin{align*}
    \mathbb{E}_{T,\alpha}^{\vrho}&\left[\overline{J}_{T,\alpha,\beta}(\vrho) \right] 
    \geq 
    \overline{J}_{T,\alpha,\beta}(\vrho) \\
    & - \sqrt{\frac{1-\delta}{\delta} 2R_{\max}^2 \left( C_\gamma(\alpha)^2 + C_\omega {B}_{T,\alpha,\beta}(\vrho)\right)}.
\end{align*}
\end{restatable}

We are now ready to construct the surrogate objective function and show how to optimize it. Following the path of~\citet{metelli2018policy}, we take an \emph{uncertainty-averse} approach, by maximizing a probabilistic lower bound of the quantity of interest, \ie the one presented in Theorem~\ref{thr:thrlbound}. Renaming $\lambda = \sqrt{\frac{1-\delta}{\delta} 2 R_{\max}^2 }$,  and treating it as a hyperparameter, we get the following surrogate objective:
\begin{align}
    \mathcal{L}_{\lambda}(\vrho) = \overline{J}_{T,\alpha,\beta}(\vrho) 
     - \lambda  \sqrt{C_\gamma(\alpha)^2+C_\omega B_{T,\alpha,\beta}(\vrho)}.
    \label{eq:objective}
\end{align}
In order to optimize this objective, we use a policy-gradient approach that is discussed in Appendix~\ref{app:fut_ret_grad}. We call this algorithm POLIS (Policy Optimization in Lifelong learning through Importance Sampling) and provide its pseudocode in \cref{alg:cap}. 

\begin{algorithm}[t]
\small
\caption{Lifelong learning with POLIS}\label{alg:cap}
\textbf{Input}: steps behind $\alpha$, steps ahead $\beta$, regularization $\lambda$, discount factor $\omega$, training period $h$, training epochs $N$,
\begin{algorithmic}[1]
\State Initialize $\nu_{\boldsymbol{\rho}}$, $t \leftarrow 0$
\While{True}  \Comment{Possibly never-ending loop}
\State Sample $\boldsymbol{\theta}_t\sim\nu_{\rho}(t)$
\State Collect new state $s_t$ and reward $r_t$ using $\pi_{\boldsymbol{\theta}_t}$
\If{$t$ mod $h = 0$ and $t>1$}
    \For{\texttt{$i\in \{1,\dots,N\}$}} \Comment{Training loop of POLIS}
        \State Compute $\widecheck{J}_{t,\alpha}({\boldsymbol{\rho}})$,  $\overline{J}_{t,\alpha,\beta}({\boldsymbol{\rho}})$ and  $B_{t,\alpha,\beta}(\boldsymbol{\rho})$ 
        \State $\boldsymbol{\rho} \leftarrow \arg\max_{\boldsymbol{\rho}} \mathcal{L}_{\lambda}(\boldsymbol{\rho})$ 
        \Comment{See \cref{eq:objective}}
     \EndFor
\EndIf
\State $t \leftarrow t+1$
\EndWhile
\end{algorithmic}
\end{algorithm}


\section{Related Works}\label{sec:relatedWorks}

The problem of lifelong RL is not new to the community.
Nevertheless, the term encapsulates problems which can have slightly differing definitions, thus hindering the comparison between existing approaches.
In this section, we will briefly discuss the most relevant ones. A more complete overview can be found in \cite{padakandla2021survey}.
Lifelong RL approaches handle finite-horizon settings \cite{hallak2015contextual,ortner2020uclr2} as well as infinite-horizon settings. 
In this second case, many works focus on the detection of abrupt changes in the dynamics \cite{daSilva2006context, hadoux2014sequential} or on scenarios where the non-stationarity arises from switching between stationary dynamics, where the number of such dynamics is known \cite{choi2000hidden}. 

Non-stationarity in the MDP dynamics is not only bound to LL, since it is also a core element of \textit{Continual Learning}.
To adapt to the evolving environment, continually learning agents need to find structure in the world to tackle new tasks by decomposing them in smaller sub-problems through function composition \cite{griffiths2019subtasks} or by extracting meaningful information in the form of abstract concepts \cite{zhang2018decoupling, Francois-Lavet2019abstract}.
Other approaches focus on capturing task-agnostic underlying dynamics of the world, by building auxiliary tasks like reward prediction \cite{jaderberg2017rew-pred} or using inverse dynamics prediction
\cite{Shelhamer2017dyn-pred} to provide denser training signal.
A general overview of continual RL approaches can be found in \cite{khetarpal2020towards}.

Another relevant approach to Lifelong RL is Meta RL, which leverages past experience to learn new skills more efficiently, i.e. using a small amount of new data. 
Usual Meta RL algorithms can be adapted to non-stationarity by modelling the consecutive tasks as a Markov chain model \cite{al-shedivat2018continuous} or using experience replay \cite{riemer2018learning}. 

Lastly, more similar to our approach, is the one of \citet{chandak2020optimizing} where the policy is trained to optimize the future predicted performance. To this end, the past performance is first of all estimated through importance sampling and then used to forecast future performance. All these steps are differentiable, which allows optimizing the policy through gradient ascent. 
The  authors  propose  two  algorithms, Pro-OLS, forecasting the performance using an ordinary least-squares regression and Pro-WLS, where the regression takes into account the importance weights inside a weighted least-squares.
The latter reduces the variance of the estimates at the expense of adding some bias.
One major difference with our approach is that their method is designed for {episodic} RL, where the non-stationarity arises from one episode to the next. We instead consider a truly lifelong framework, where that are no episodes and non-stationarity arises at the single step level.
Our approach is different for three other reasons. First, 
while our estimate of the $\beta$-step ahead performance has in common with the aforementioned paper the use of importance sampling, our objective is however greatly different as 
we add an extra discounting parameter to control the bias due to non-stationarity and two terms, the $\alpha$-step behind performance and a variance regularization.
Second,
our surrogate objective can be optimized at {any point} in time, meaning that if a significant shift in dynamics is detected, one has the opportunity to retrain the algorithm suddenly.
Third, 
we consider a parameter-based approach in which the hyper-policy depends only upon time, the policy may thus change at every step. \citet{chandak2020optimizing} consider an action-based approach where policy's parameters are fixed during an episode.
\section{Experiments}\label{sec:experiments}
In this section, we report the experimental evaluation of our algorithm in comparison with state-of-the-art baselines.

\subsection{Lifelong Learning Framework}
The schedule for a lifelong interaction with the environment is divided in two periods. In the first, which we refer to as \textit{behavioural period}, a behavioural hyper-policy is queried to sample data in order to gather enough samples to compute the first {$\alpha$-step behind expected return}.
In the second period, referred to as \textit{target period}, the agent continues interacting with same environment, but is now training its hyper-policy every few steps (50 in all experiments) for a given number of gradient steps (100 in all experiments).
For all tasks, we set $\gamma=\omega=1$.\footnote{The code is available at \url{https://github.com/pierresdr/polis}.} 

We consider a particular subclass of non-stationary environments, frequently encountered in practice. The state $x = (x^c,x^u)$ is decomposed into a controllable $x^c$ and a non-controllable $x^u$ part. The controllable part evolves according to stationary dynamics and depends on the action $P^c((x')^c|x^c,a)$. The non-controllable part instead is not affected by the action and follows non-stationary dynamics $P_t^u((x')^u|x^u)$.
\begin{ass}
	\label{ass:factorable_state}
The transition model $P= (P_t)_{t \in \Nat}$ factorizes as follows for every $x = (x^c,x^u) \in \xs$, $a \in \As$, and $t \in \Nat$:
\begin{equation}
	P_t(x'|x,a) = P^c((x')^c|x^c,x^u,a) P_t^u((x')^u|x^u) .
\end{equation}
\end{ass}
Under the assumption, we can sample new trajectories from the last $\alpha$ steps, where the non-stationary part of the state $x^u$ is kept fixed but the sampled policy at each time and therefore the stationary controllable part of the state $x^c$ changes. Therefore, we have access to the value of the gradient of the $\alpha$-step behind expected return by direct estimation, without requiring importance sampling. 

\subsection{Trading Environment}
The first task is the daily trading of the EUR-USD (€/\$) currency pair from the Foreign Exchange (Forex).
Following \cite{bisi2020foreign}, we allow the agent to trade up to a fixed quantity of 100k\$ USD with a per-transaction fee $f$ of 1\$. 
The agent has a continuous actions space in $[-1,1]$, where $1$ and $-1$ correspond to buy or sell with the maximum order size, while $0$ corresponds to staying flat.
We do not model the effect of our agent's trades on the market,\footnote{We assume that the order size of the agent is negligible \wrt the market liquidity.} thus satisfying \cref{ass:factorable_state}.
The state of the agent is composed of its actual portfolio ($x_{t}^{c}$, which also corresponds to its previous action) and the current rate of the currency ($x_{t}^{u}$).
The reward is defined as $r_{t}=a_t(x_{t+1}^{u}-x_{t}^{u})-f\lvert a_t - x_{t}^{c}\rvert$.

We consider three datasets of historical data, 2009-2012, 2013-2016, and 2017-2020; each period having a little more than $1000$ data points. $\alpha$ is set to 500 and we consider a {target period} of 500 steps.

Finding a satisfactory set of hyperparameters (in the sense of parameters of the algorithm itself, not parameters of the hyper-policy) can be problematic in our lifelong scenario. Indeed, here, there is no distinction between training and testing since the parameters are continually updated. Selecting hyperparameters for future interactions with the environment by evaluating past performances is thus prone to overfitting on the past performance. 
To account for this problem, we compare two hyperparameter selection approaches.
In the first, we select the best performing hyperparameters from the dataset 2009-2012 and evaluate the selection on the other two datasets.
In the second approach, we both select the hyperparameters and evaluate on the last two datasets.

The trading of the EUR-USD currency pair is a highly complex task. To give more chance to the algorithm to exploit potential patterns of the series, we provide another trading task on a simulated series. The framework is the same, only changes the underlying rate process which will now be a Vasicek process. In this scenario, the rate $(p_t)_{t\geq1}$ satisfies $ p_{t+1} = 0.9 p_t + u_t$, where $u_t\sim\mathcal{N}(0,1)$. On this task we will test the set of hyperparameters selected on the EUR-USD.

\subsection{Dam Environment}
The second environment is a water resource management problem. A dam is used to save water from rains and possibly release it to meet a certain demand for water (\eg the needs of a town).
We model the environment following \cite{castelletti2010tree, pmlr-v80-tirinzoni18a}.
The inflow (\eg rain) is the non-stationary process and the agent has obviously no impact on it, thus satisfying \cref{ass:factorable_state}. The mean inflow follows one of either 3 profiles given in \cref{app:dam_dataset}. 
The state observed by the agent is the day's lake level. The agent does not observe the day of the year, contrarily to \cite{pmlr-v80-tirinzoni18a}, in order to ensure non-stationarity. 
The agent's action is continuous and corresponds to selecting the daily amount of water to release in order to avoid flooding and meet the demand.
Considering the flooding level $F=300$ and the daily demand for water $D=10$, the penalty that the agent gets for each is respectively $c_F = (\max(x-F,0))^2$ and $c_D = (\max(a-D,0))^2$, where $a$ is the action of the agent and $x$ the current lake level.
The final cost is a convex combination of those costs, whose weights depend on the inflow profile (see \cref{app:dam_dataset}). In this environment, $\alpha$ is set to 1000 in order to include enough years of past data in the estimator. We provide results for a {target period} of 500 steps. Because the results for this environment are less sensitive to the choice of hyperparameters, we only select them according to the performance given the first inflow profile.

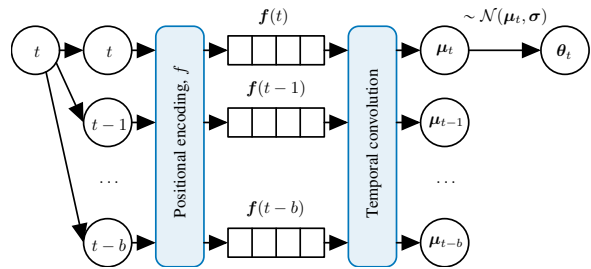
\begin{figure}
    \centering
    \usetikzlibrary{arrows}

\def\x{0}
\def\y{0}
\def\scale{0.32}

\begin{tikzpicture}[line cap=round,line join=round,>=triangle 45,x=1cm,y=1cm]


\clip(\x,\y) rectangle (25*\scale+\x,11.5*\scale+\y);

\draw [line width=\scale*2pt] (1*\scale+\x,9*\scale+\y) circle (\scale*1cm);
\draw [->,line width=\scale*2pt] (1.44*\scale+\x,8.1*\scale+\y) -- (3*\scale+\x,1*\scale+\y);
\draw [->,line width=\scale*2pt] (1.86*\scale+\x,8.49*\scale+\y) -- (3*\scale+\x,6*\scale+\y);

\draw (0.7*\scale+\x,9.4*\scale+\y) node[anchor=north west,scale=\scale*2] {$t$};
\draw (3.7*\scale+\x,9.4*\scale+\y) node[anchor=north west,scale=\scale*2] {$t$};
\draw (3.1*\scale+\x,6.4*\scale+\y) node[anchor=north west,scale=\scale*2] {$t-1$};
\draw (3.1*\scale+\x,1.4*\scale+\y) node[anchor=north west,scale=\scale*2] {$t-b$};
\draw (3.4*\scale+\x,3.8*\scale+\y) node[anchor=north west,scale=\scale*2] {$\dots$};

\draw[line width=\scale*2pt,rounded corners,color=vibrantBlue,fill=vibrantBlue,fill opacity=0.1] (14*\scale+\x,0*\scale+\y) rectangle (16*\scale+\x,10*\scale+\y) {};
\draw (14.5*\scale+\x,1.8*\scale+\y) node[anchor=north west,scale=\scale*2,rotate=90] {Temporal convolution};

\draw[line width=\scale*2pt,rounded corners,color=vibrantBlue,fill=vibrantBlue,fill opacity=0.1] (6*\scale+\x,0*\scale+\y) rectangle (8*\scale+\x,10*\scale+\y) {};
\draw (6.5*\scale+\x,1.8*\scale+\y) node[anchor=north west,scale=\scale*2,rotate=90] {Positional encoding, $f$};

\draw [->,line width=\scale*2pt] (2*\scale+\x,9*\scale+\y) -- (3*\scale+\x,9*\scale+\y);
\draw [->,line width=\scale*2pt] (5*\scale+\x,9*\scale+\y) -- (6*\scale+\x,9*\scale+\y);
\draw [->,line width=\scale*2pt] (5*\scale+\x,6*\scale+\y) -- (6*\scale+\x,6*\scale+\y);
\draw [->,line width=\scale*2pt] (5*\scale+\x,1*\scale+\y) -- (6*\scale+\x,1*\scale+\y);


\draw[line width=\scale*2pt,] (9*\scale+\x,9.6*\scale+\y) rectangle (13*\scale+\x,8.4*\scale+\y) {};
\draw [line width=\scale*2pt] (10*\scale+\x,8.4*\scale+\y) -- (10*\scale+\x,9.6*\scale+\y);
\draw [line width=\scale*2pt] (11*\scale+\x,8.4*\scale+\y) -- (11*\scale+\x,9.6*\scale+\y);
\draw [line width=\scale*2pt] (12*\scale+\x,8.4*\scale+\y) -- (12*\scale+\x,9.6*\scale+\y);
\draw (10*\scale+\x,11*\scale+\y) node[anchor=north west,scale=\scale*2,] {$\boldsymbol{f}(t)$};

\draw[line width=\scale*2pt,] (9*\scale+\x,6.6*\scale+\y) rectangle (13*\scale+\x,5.4*\scale+\y) {};
\draw [line width=\scale*2pt] (10*\scale+\x,5.4*\scale+\y) -- (10*\scale+\x,6.6*\scale+\y);
\draw [line width=\scale*2pt] (11*\scale+\x,5.4*\scale+\y) -- (11*\scale+\x,6.6*\scale+\y);
\draw [line width=\scale*2pt] (12*\scale+\x,5.4*\scale+\y) -- (12*\scale+\x,6.6*\scale+\y);
\draw (9.5*\scale+\x,8*\scale+\y) node[anchor=north west,scale=\scale*2,] {$\boldsymbol{f}(t-1)$};

\draw[line width=\scale*2pt,] (9*\scale+\x,1.6*\scale+\y) rectangle (13*\scale+\x,0.4*\scale+\y) {};
\draw [line width=\scale*2pt] (10*\scale+\x,0.4*\scale+\y) -- (10*\scale+\x,1.6*\scale+\y);
\draw [line width=\scale*2pt] (11*\scale+\x,0.4*\scale+\y) -- (11*\scale+\x,1.6*\scale+\y);
\draw [line width=\scale*2pt] (12*\scale+\x,0.4*\scale+\y) -- (12*\scale+\x,1.6*\scale+\y);
\draw (9.5*\scale+\x,3*\scale+\y) node[anchor=north west,scale=\scale*2,] {$\boldsymbol{f}(t-b)$};

\draw [->,line width=\scale*2pt] (8*\scale+\x,9*\scale+\y) -- (9*\scale+\x,9*\scale+\y);
\draw [->,line width=\scale*2pt] (8*\scale+\x,6*\scale+\y) -- (9*\scale+\x,6*\scale+\y);
\draw [->,line width=\scale*2pt] (8*\scale+\x,1*\scale+\y) -- (9*\scale+\x,1*\scale+\y);
\draw [->,line width=\scale*2pt] (13*\scale+\x,9*\scale+\y) -- (14*\scale+\x,9*\scale+\y);
\draw [->,line width=\scale*2pt] (16*\scale+\x,9*\scale+\y) -- (17*\scale+\x,9*\scale+\y);
\draw [->,line width=\scale*2pt] (13*\scale+\x,6*\scale+\y) -- (14*\scale+\x,6*\scale+\y);
\draw [->,line width=\scale*2pt] (16*\scale+\x,6*\scale+\y) -- (17*\scale+\x,6*\scale+\y);
\draw [->,line width=\scale*2pt] (13*\scale+\x,1*\scale+\y) -- (14*\scale+\x,1*\scale+\y);
\draw [->,line width=\scale*2pt] (16*\scale+\x,1*\scale+\y) -- (17*\scale+\x,1*\scale+\y);

\draw [line width=\scale*2pt] (4*\scale+\x,9*\scale+\y) circle (\scale*1cm);
\draw [line width=\scale*2pt] (18*\scale+\x,9*\scale+\y) circle (\scale*1cm);
\draw [line width=\scale*2pt] (4*\scale+\x,6*\scale+\y) circle (\scale*1cm);
\draw [line width=\scale*2pt] (18*\scale+\x,6*\scale+\y) circle (\scale*1cm);
\draw [line width=\scale*2pt] (4*\scale+\x,1*\scale+\y) circle (\scale*1cm);
\draw [line width=\scale*2pt] (18*\scale+\x,1*\scale+\y) circle (\scale*1cm);

\draw (17.1*\scale+\x,1.5*\scale+\y) node[anchor=north west,scale=\scale*2] {$\boldsymbol{\mu}_{t-b}$};
\draw (17.1*\scale+\x,6.5*\scale+\y) node[anchor=north west,scale=\scale*2] {$\boldsymbol{\mu}_{t-1}$};
\draw (17.4*\scale+\x,9.5*\scale+\y) node[anchor=north west,scale=\scale*2] {$\boldsymbol{\mu}_{t}$};
\draw (17.4*\scale+\x,3.8*\scale+\y) node[anchor=north west,scale=\scale*2] {$\dots$};

\draw [->,line width=\scale*2pt] (19*\scale+\x,9*\scale+\y) -- (22*\scale+\x,9*\scale+\y);

\draw (22.4*\scale+\x,9.5*\scale+\y) node[anchor=north west,scale=\scale*2] {$\vtheta_{t}$};

\draw [line width=\scale*2pt] (23*\scale+\x,9*\scale+\y) circle (\scale*1cm);

\draw (18.5*\scale+\x,11*\scale+\y) node[anchor=north west,scale=\scale*2,] {$\sim\mathcal{N}(\boldsymbol{\mu}_{t}, \boldsymbol{\sigma})$};

\end{tikzpicture}
    \caption{The hyper-policy queried at time $t$. First the last $b$ times are appended to $t$, $b$ the {receptive field} of the {temporal convolution}. They are fed to a positional encoding which outputs a vector $\boldsymbol{f}(t)$ for each. The latter are then fed to a {temporal convolution} which returns the mean $\boldsymbol{\mu}_t$ of $\vtheta_t$.}
    \label{fg:hyper_policy}
\end{figure}

\begin{figure*}[t]
\centering
\begin{minipage}{.62\textwidth}
  \centering
  \begin{subfigure}{.49\textwidth}
  \centering
  \includegraphics[width=\linewidth]{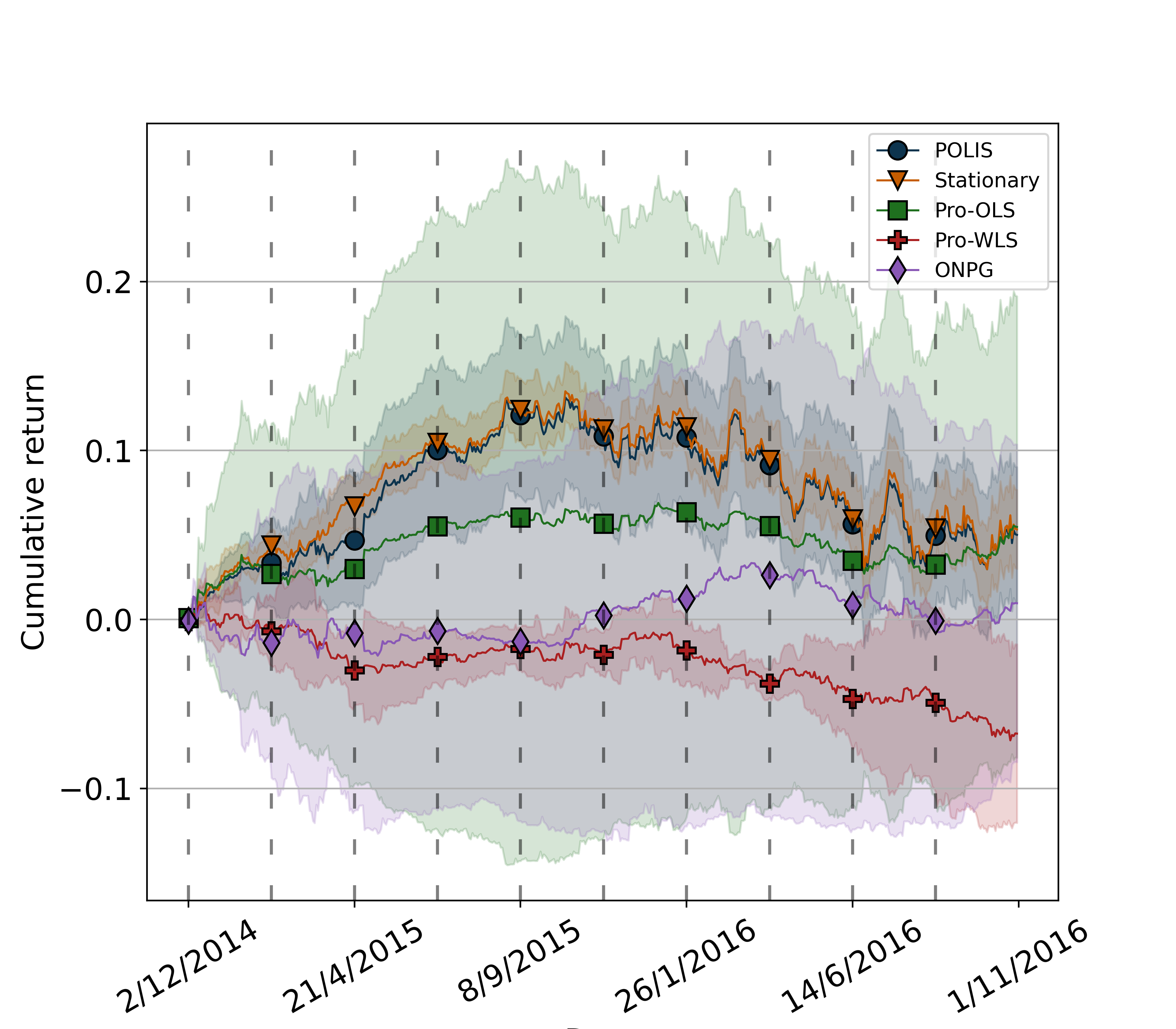}
\end{subfigure}%
\hfill
\begin{subfigure}{.49\textwidth}
  \centering
  \includegraphics[width=\linewidth]{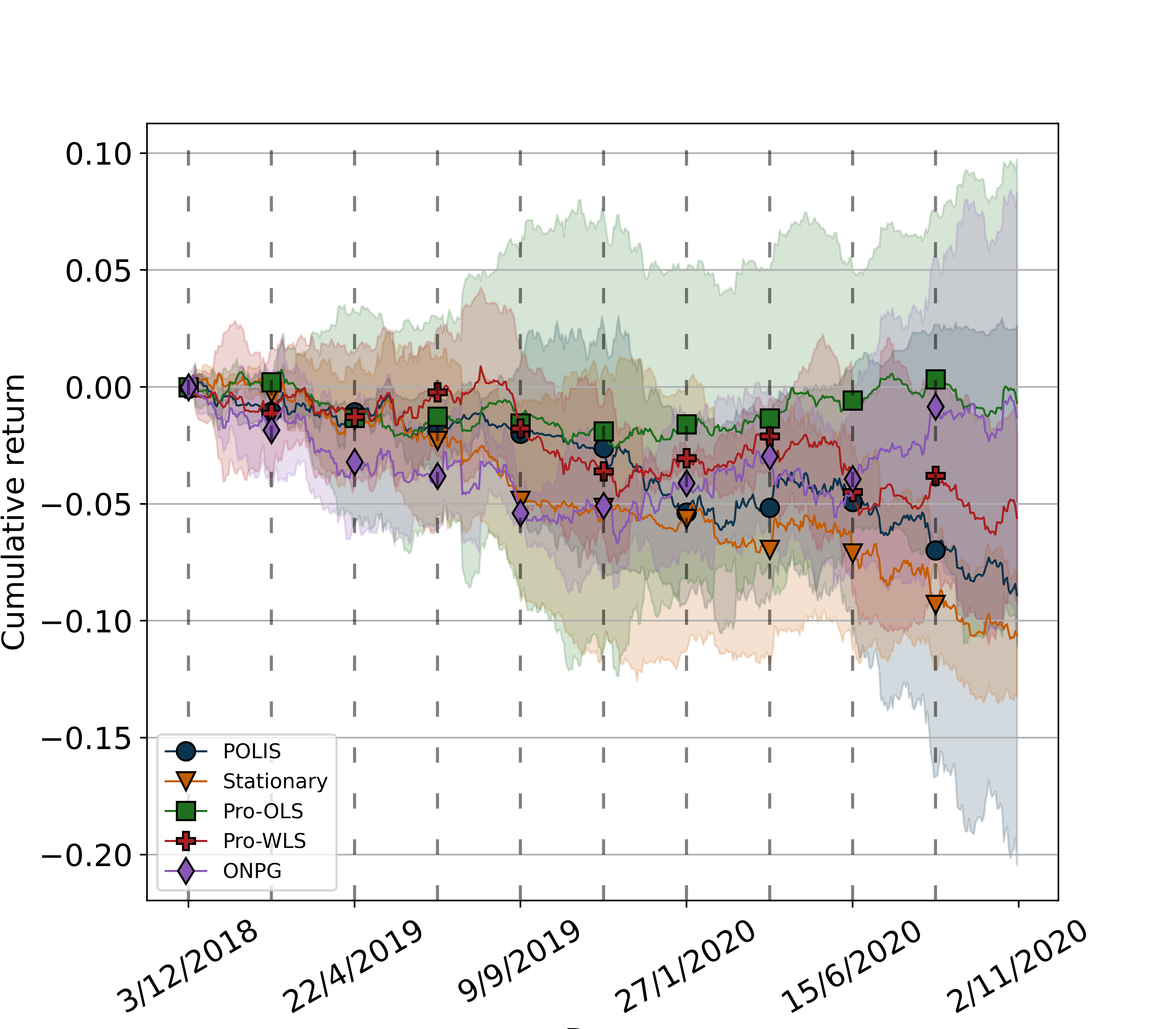}
\end{subfigure}
\captionof{figure}{Lifelong learning on the EUR-USD currency pair (left 2013-2016, right 2017-2020). Mean cumulative returns on the target period with one standard deviation shaded area, over 3 seeds. Vertical dashed line indicate retrain. Hyperparameters selected on the return of 2009-2012.} \label{fig:exp_trading_T1}
\end{minipage}%
\hfill
\begin{minipage}{.36\textwidth}
  \centering
  \includegraphics[width=.843\linewidth]{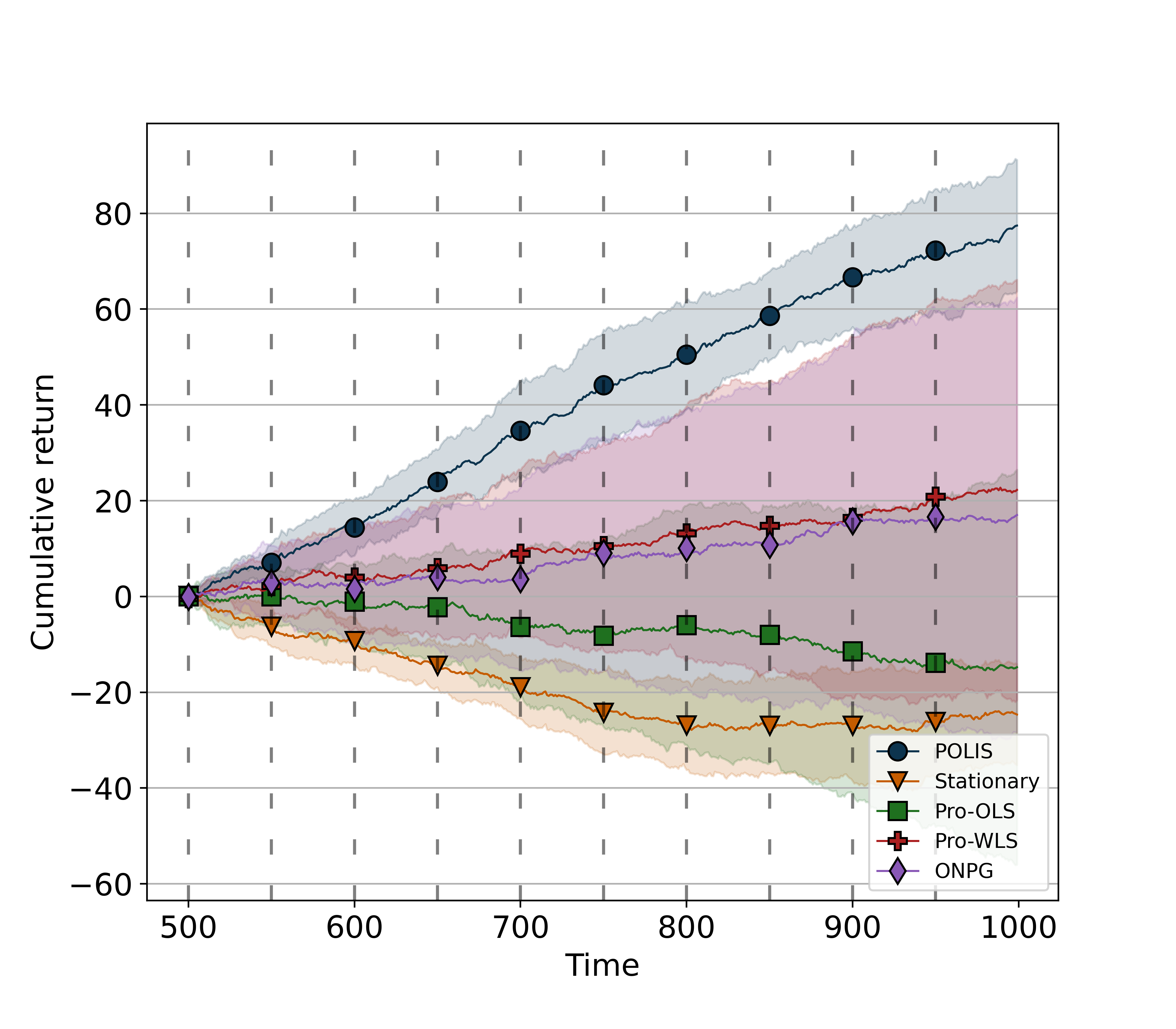}
  \captionof{figure}{Lifelong learning on the Vasicek process. Mean cumulative returns on the target period with one standard deviation shaded area, over 10 seeds. Vertical dashed line indicate retrain.}\label{fig:exp_vasicek}
  \label{fig:test2}
\end{minipage}
\end{figure*}

\begin{figure*}[t]
\centering
\begin{minipage}{.62\textwidth}
  \centering
  \begin{subfigure}{.49\textwidth}
  \centering
  \includegraphics[width=\linewidth]{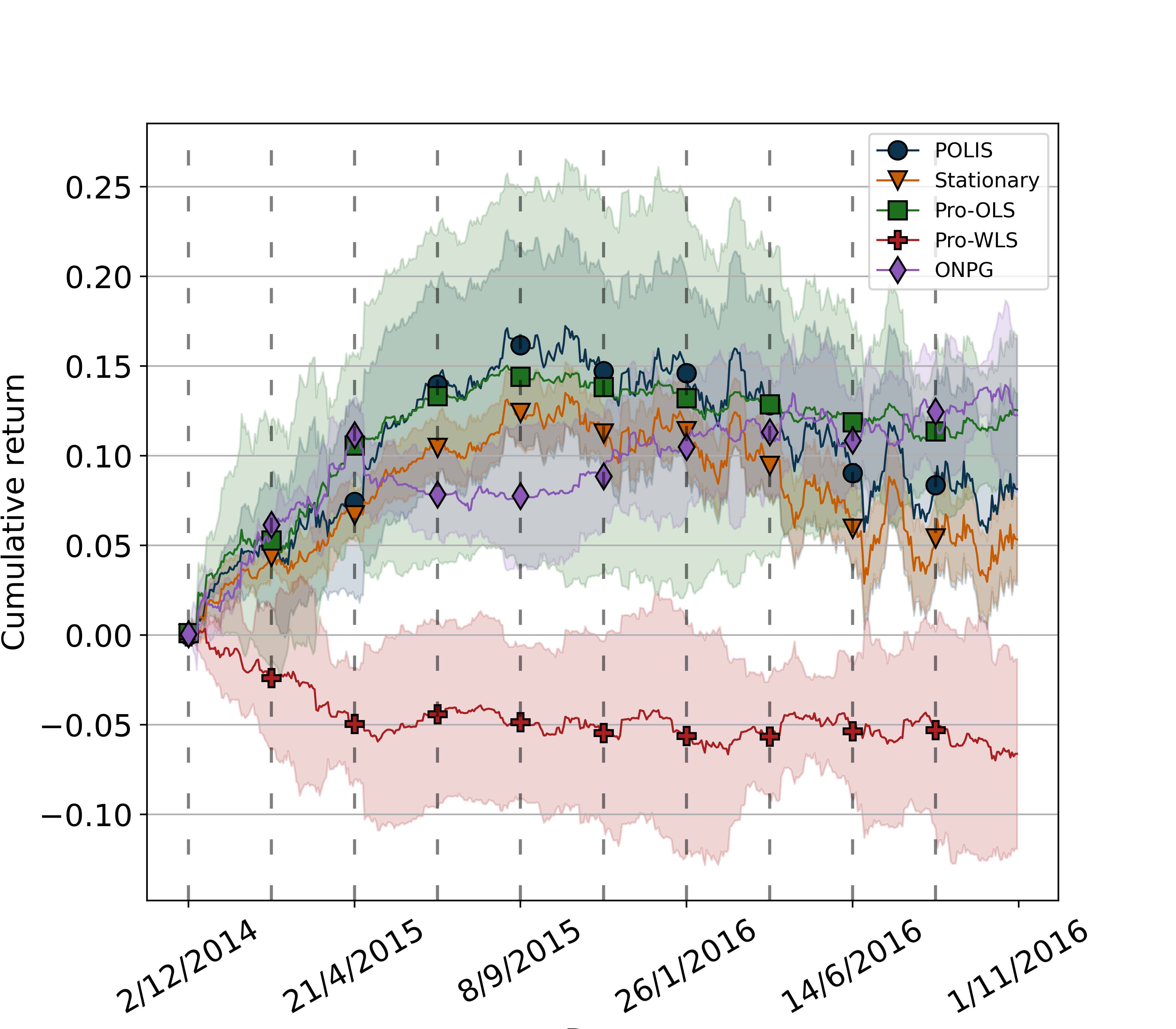}
\end{subfigure}%
\hfill
\begin{subfigure}{.49\textwidth}
  \centering
  \includegraphics[width=\linewidth]{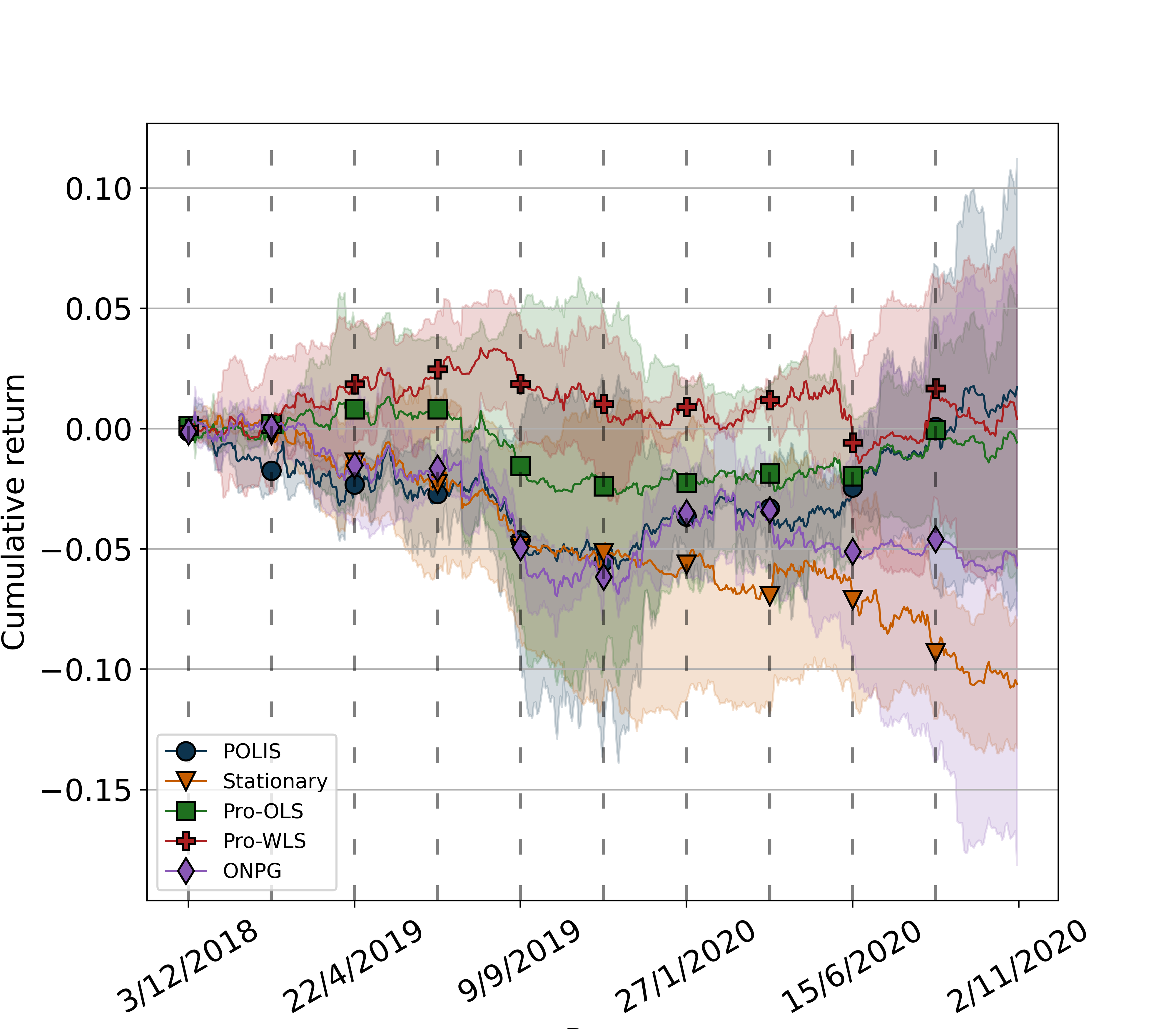}
\end{subfigure}
\captionof{figure}{Lifelong learning on the EUR-USD currency pair (left 2013-2016, right 2017-2020). Mean cumulative returns on the target period with one standard deviation shaded area, over 3 seeds. Vertical dashed line indicate retrain. Hyperparameters selected on the returns of 2013-2016 and 2017-2020.}\label{fig:exp_trading_T2T3}
\end{minipage}%
\hfill
\begin{minipage}{.36\textwidth}
  \centering
  \resizebox{\textwidth}{!}{%
    \begin{tabular}{l|lll}
          & \multicolumn{1}{c}{Inflow 1}       & \multicolumn{1}{c}{Inflow 2}        & \multicolumn{1}{c}{Inflow 3}       
        \\ \hline\hline
        POLIS&$-2.2 \pm 0.2$&$-1.5 \pm 0.0$&$-3.2 \pm 0.2$\\Stationary&$-2.2 \pm 0.1$&$-1.5 \pm 0.0$&$-3.2 \pm 0.2$\\ONPG&$-5.1 \pm 3.2$&$-1.4 \pm 0.2$&$-4.1 \pm 0.5$\\Pro-OLS&$-2.6 \pm 0.4$&$-5.2 \pm 5.1$&$-3.8 \pm 0.7$\\Pro-WLS&$-5.5 \pm 4.3$&$-8.5 \pm 9.7$&$-8.4 \pm 4.2$\\
    \end{tabular}}
    \captionof{table}{Lifelong learning on the Dam environment for each of 3 inflow profile. Mean return on the target period and standard deviation over 3 seeds. Reported results are divided by an order of $1e3$ for aesthetic.}
    \label{tab:exp_dam}
\end{minipage}
\end{figure*}

\subsection{The hyper-policy and policy}
We now describe the hyper-policy used in all experiments. It is composed of two modules.
The first is \textit{positional encoding} introduced in \cite{vaswani2017attention}. It embeds its input, time, as a vector of Fourier basis. Therefore, it does not add learnable parameters to the hyper-policy. This module has two main advantages. First, its output dimension is free to be chosen, allowing to control the input size of the next module. Second, while time eventually becomes large, the output of positional encoding is bounded, which is a valuable property when then fed to a neural network. 
The second module are convolutions scanning through time. We chose convolutions as they generally excel in finding patterns in time series. Moreover, they allow processing inputs of variable length and are easily parallelizable. 
We use a particular type of convolutions, \textit{temporal convolutions}~\cite{oord2016wavenet} which preserve time causality. 
Obviously, the convolutions require several time-steps in order to scan through with their kernel. However our hyper-policy takes only the current time as input, $\nu_{\vrho}(\cdot|t)$. 
Nevertheless, we can freely decide to consider the positional encoding of $t$ and a few previous times to reach the length of the \textit{receptive field}. 
Its length is $b=2^{l-1}(k-1)$, where $l$ is the number of layers and $k$ the kernel size of the temporal convolution.
Another advantage of using temporal convolutions is that the computation of the policy parameters $\vtheta$ can be made in parallel. This is an interesting property in practice as between two updates of the hyper-policy, we can already sample in parallel all the policy parameters to be used. The output of the temporal convolutions is the mean $\boldsymbol{\mu}_t$ of the normally distributed policy parameter $\vtheta_t$. The standard deviation of each entry of $\vtheta_t$ is not time dependent and can be either learned or fixed during training. A schematic representation of the hyper-policy is given in \cref{fg:hyper_policy}.

At the policy level, in all the experiments, we use an affine policy with bounded outputs.

\subsection{Baselines}
The first baseline is the stationary hyper-policy which can be seen as a special case of POLIS when $\nu_{\vrho}(\cdot|t)=\nu_{\vrho}(\cdot)$. We consider this hyper-policy along with the same affine policy.  
Note that, although stationary between trainings, the hyper-policy's parameters are retrained every 50 steps.

We then consider baselines from the literature, including Pro-OLS and Pro-WLS \cite{chandak2020optimizing}. 
In their experiments, \citet{chandak2020optimizing} use a baseline which they refer to as ONPG, replicating the idea of \citet{al-shedivat2018continuous}. We also include this baseline and thank \citet{chandak2020optimizing} for providing their code.

\subsection{Results}
\textbf{Trading environment}~~The cumulative returns obtained for the hyperparameters selected on 2009-2012 are given in \cref{fig:exp_trading_T1}.  Interestingly, on the period 2013-2016, POLIS has a  performance similar to the stationary policy, which is comparable or superior to baselines. 
On the period 2017-2020, POLIS  under-performs the baselines, but the stationary one.
When selecting the set of hyperparameters from the testing dataset, we obtain the results shown in  \cref{fig:exp_trading_T2T3}. This time, POLIS obtains more similar performance to the baselines, closing the gap on the period 2017-2020.

The cumulative returns obtained for the trading on the Vasicek process are reported in \cref{fig:exp_vasicek}. 
On this tasks, specifically designed to highlight the smooth non-stationarity, POLIS is clearly superior to baselines, particularly the stationary one. Note also its smaller variance.

\textbf{Dam environment}~~We report the results of the experiment in \cref{tab:exp_dam}. Surprisingly, the stationary policy is able to significantly outperform the other baselines for each inflow profile, with the exception of ONPG on the second profile. However, we note the higher standard deviation of ONPG in this case.
Remarkably, POLIS is able to match the stationary policy's performance on each inflow profile. 
This indicates that our approach is able to avoid extra non-stationarity in tasks where it is not needed. 

\section{Conclusion}
In this paper, we proposed to address the \textit{lifelong} RL problem by using a hyper-policy mapping time to policy parameters. To grasp the objective of LL, \ie the future performance, we designed an estimator of such quantity, making use of the past collected experience via importance sampling. The estimator has a controllable bias which vanishes as the environment and the hyper-policy become stationary.
Besides, we add two terms to the objective: an estimation of the past performance preventing catastrophic forgetting and a penalization based on an upper-bound on the variance, which prevents overfitting the past and favors generalization to future non-stationarity. 
We proposed an implementation of such hyper-policy which we tested in several scenarios, demonstrating that our approach can exploit predictable non-stationarity, control for its variance and avoid excessive non-stationarity when non necessary.
Our approach tackled exploration via the stochasticity of the hyper-policy. Future work include a more principled and explicit treatment of the exploration problem in the lifelong RL setting.

{\small
\bibliography{biblio}

\begin{thebibliography}{39}
\providecommand{\natexlab}[1]{#1}

\bibitem[{Al-Shedivat et~al.(2018)Al-Shedivat, Bansal, Burda, Sutskever,
  Mordatch, and Abbeel}]{al-shedivat2018continuous}
Al-Shedivat, M.; Bansal, T.; Burda, Y.; Sutskever, I.; Mordatch, I.; and
  Abbeel, P. 2018.
\newblock Continuous Adaptation via Meta-Learning in Nonstationary and
  Competitive Environments.
\newblock In \emph{ICLR 2018}.

\bibitem[{Bisi et~al.(2020)Bisi, Liotet, Sabbioni, Reho, Montali, Corno, and
  Restelli}]{bisi2020foreign}
Bisi, L.; Liotet, P.; Sabbioni, L.; Reho, G.; Montali, N.; Corno, C.; and
  Restelli, M. 2020.
\newblock Foreign Exchange Trading: A Risk-Averse Batch Reinforcement Learning
  Approach.
\newblock In \emph{ICAIF ’20}.

\bibitem[{Bowerman(1974)}]{bowerman1974nonstationary}
Bowerman, B.~L. 1974.
\newblock Nonstationary Markov decision processes and related topics in
  nonstationary Markov chains.

\bibitem[{Brunskill and Li(2014)}]{brunskill2014pac}
Brunskill, E.; and Li, L. 2014.
\newblock PAC-inspired Option Discovery in Lifelong Reinforcement Learning.
\newblock In \emph{ICML 2014}.

\bibitem[{Castelletti et~al.(2010)Castelletti, Galelli, Restelli, and
  Soncini-Sessa}]{castelletti2010tree}
Castelletti, A.; Galelli, S.; Restelli, M.; and Soncini-Sessa, R. 2010.
\newblock Tree-based reinforcement learning for optimal water reservoir
  operation.
\newblock \emph{Water Resources Research}, 46(9).

\bibitem[{Chandak et~al.(2020)Chandak, Theocharous, Shankar, Mahadevan, White,
  and Thomas}]{chandak2020optimizing}
Chandak, Y.; Theocharous, G.; Shankar, S.; Mahadevan, S.; White, M.; and
  Thomas, P.~S. 2020.
\newblock Optimizing for the Future in Non-Stationary MDPs.
\newblock \emph{ICML 2020}.

\bibitem[{Choi, Yeung, and Zhang(2000)}]{choi2000hidden}
Choi, S.~P.; Yeung, D.-Y.; and Zhang, N.~L. 2000.
\newblock Hidden-mode markov decision processes for nonstationary sequential
  decision making.
\newblock In \emph{Sequence Learning}, 264--287. Springer.

\bibitem[{da~Silva et~al.(2006)da~Silva, Basso, Bazzan, and
  Engel}]{daSilva2006context}
da~Silva, B.~C.; Basso, E.~W.; Bazzan, A. L.~C.; and Engel, P.~M. 2006.
\newblock Dealing with Non-Stationary Environments Using Context Detection.
\newblock In \emph{ICML 2006}.

\bibitem[{Deisenroth, Neumann, and Peters(2013)}]{deisenroth2013survey}
Deisenroth, M.~P.; Neumann, G.; and Peters, J. 2013.
\newblock A Survey on Policy Search for Robotics.
\newblock \emph{Foundations and Trends in Robotics}, 2(1-2).

\bibitem[{Francois-Lavet et~al.(2019)Francois-Lavet, Bengio, Precup, and
  Pineau}]{Francois-Lavet2019abstract}
Francois-Lavet, V.; Bengio, Y.; Precup, D.; and Pineau, J. 2019.
\newblock Combined Reinforcement Learning via Abstract Representations.
\newblock In \emph{AAAI 2019}.

\bibitem[{Garcia and Smith(2000)}]{garcia2000solving}
Garcia, A.; and Smith, R.~L. 2000.
\newblock Solving nonstationary infinite horizon stochastic production planning
  problems.
\newblock \emph{Oper. Res. Lett.}, 27(3).

\bibitem[{Ghate and Smith(2013)}]{ghate2013linear}
Ghate, A.; and Smith, R.~L. 2013.
\newblock A Linear Programming Approach to Nonstationary Infinite-Horizon
  Markov Decision Processes.
\newblock \emph{Oper. Res.}, 61(2): 413--425.

\bibitem[{Griffiths et~al.(2019)Griffiths, Callaway, Chang, Grant, Krueger, and
  Lieder}]{griffiths2019subtasks}
Griffiths, T.~L.; Callaway, F.; Chang, M.~B.; Grant, E.; Krueger, P.~M.; and
  Lieder, F. 2019.
\newblock Doing more with less: meta-reasoning and meta-learning in humans and
  machines.
\newblock \emph{Current Opinion in Behavioral Sciences}, 29: 24--30.
\newblock Artificial Intelligence.

\bibitem[{Hadoux, Beynier, and Weng(2014)}]{hadoux2014sequential}
Hadoux, E.; Beynier, A.; and Weng, P. 2014.
\newblock Sequential decision-making under non-stationary environments via
  sequential change-point detection.
\newblock In \emph{Learning over multiple contexts (LMCE)}.

\bibitem[{Hallak, Di~Castro, and Mannor(2015)}]{hallak2015contextual}
Hallak, A.; Di~Castro, D.; and Mannor, S. 2015.
\newblock Contextual markov decision processes.
\newblock \emph{arXiv preprint arXiv:1502.02259}.

\bibitem[{Jaderberg et~al.(2017)Jaderberg, Mnih, Czarnecki, Schaul, Leibo,
  Silver, and Kavukcuoglu}]{jaderberg2017rew-pred}
Jaderberg, M.; Mnih, V.; Czarnecki, W.~M.; Schaul, T.; Leibo, J.~Z.; Silver,
  D.; and Kavukcuoglu, K. 2017.
\newblock Reinforcement Learning with Unsupervised Auxiliary Tasks.
\newblock In \emph{ICLR 2017}.

\bibitem[{Jagerman, Markov, and de~Rijke(2019)}]{jagerman2019people}
Jagerman, R.; Markov, I.; and de~Rijke, M. 2019.
\newblock When people change their mind: Off-policy evaluation in
  non-stationary recommendation environments.
\newblock In \emph{Proceedings of the Twelfth ACM International Conference on
  Web Search and Data Mining}, 447--455.

\bibitem[{Khetarpal et~al.(2020)Khetarpal, Riemer, Rish, and
  Precup}]{khetarpal2020towards}
Khetarpal, K.; Riemer, M.; Rish, I.; and Precup, D. 2020.
\newblock Towards continual reinforcement learning: A review and perspectives.
\newblock \emph{arXiv preprint arXiv:2012.13490}.

\bibitem[{Lesner and Scherrer(2015)}]{lesner2015nonstationary}
Lesner, B.; and Scherrer, B. 2015.
\newblock Non-Stationary Approximate Modified Policy Iteration.
\newblock In Bach, F.~R.; and Blei, D.~M., eds., \emph{ICML 2015}.

\bibitem[{Metelli et~al.(2018)Metelli, Papini, Faccio, and
  Restelli}]{metelli2018policy}
Metelli, A.~M.; Papini, M.; Faccio, F.; and Restelli, M. 2018.
\newblock Policy optimization via importance sampling.
\newblock In \emph{NeurIPS 2018}.

\bibitem[{Metelli et~al.(2020)Metelli, Papini, Montali, and
  Restelli}]{metelli2020importance}
Metelli, A.~M.; Papini, M.; Montali, N.; and Restelli, M. 2020.
\newblock Importance Sampling Techniques for Policy Optimization.
\newblock \emph{Journal of Machine Learning Research}, 21(141): 1--75.

\bibitem[{Oord et~al.(2016)Oord, Dieleman, Zen, Simonyan, Vinyals, Graves,
  Kalchbrenner, Senior, and Kavukcuoglu}]{oord2016wavenet}
Oord, A. v.~d.; Dieleman, S.; Zen, H.; Simonyan, K.; Vinyals, O.; Graves, A.;
  Kalchbrenner, N.; Senior, A.; and Kavukcuoglu, K. 2016.
\newblock Wavenet: A generative model for raw audio.
\newblock \emph{arXiv preprint arXiv:1609.03499}.

\bibitem[{Ortner, Gajane, and Auer(2020)}]{ortner2020uclr2}
Ortner, R.; Gajane, P.; and Auer, P. 2020.
\newblock Variational Regret Bounds for Reinforcement Learning.
\newblock In \emph{UAI 2020}.

\bibitem[{Owen(2013)}]{mcbook}
Owen, A.~B. 2013.
\newblock \emph{Monte Carlo theory, methods and examples}.

\bibitem[{Padakandla(2021)}]{padakandla2021survey}
Padakandla, S. 2021.
\newblock A survey of reinforcement learning algorithms for dynamically varying
  environments.
\newblock \emph{ACM Computing Surveys (CSUR)}, 54(6): 1--25.

\bibitem[{Papini et~al.(2019)Papini, Metelli, Lupo, and
  Restelli}]{papini2019optimistic}
Papini, M.; Metelli, A.~M.; Lupo, L.; and Restelli, M. 2019.
\newblock Optimistic policy optimization via multiple importance sampling.
\newblock In \emph{ICML 2019}.

\bibitem[{Paszke et~al.(2019)Paszke, Gross, Massa, Lerer, Bradbury, Chanan,
  Killeen, Lin, Gimelshein, Antiga, Desmaison, Kopf, Yang, DeVito, Raison,
  Tejani, Chilamkurthy, Steiner, Fang, Bai, and Chintala}]{NEURIPS2019_9015}
Paszke, A.; Gross, S.; Massa, F.; Lerer, A.; Bradbury, J.; Chanan, G.; Killeen,
  T.; Lin, Z.; Gimelshein, N.; Antiga, L.; Desmaison, A.; Kopf, A.; Yang, E.;
  DeVito, Z.; Raison, M.; Tejani, A.; Chilamkurthy, S.; Steiner, B.; Fang, L.;
  Bai, J.; and Chintala, S. 2019.
\newblock PyTorch: An Imperative Style, High-Performance Deep Learning Library.
\newblock In Wallach, H.; Larochelle, H.; Beygelzimer, A.; d\textquotesingle
  Alch\'{e}-Buc, F.; Fox, E.; and Garnett, R., eds., \emph{Advances in Neural
  Information Processing Systems 32}, 8024--8035. Curran Associates, Inc.

\bibitem[{Puterman(2014)}]{puterman2014markov}
Puterman, M.~L. 2014.
\newblock \emph{Markov decision processes: discrete stochastic dynamic
  programming}.
\newblock John Wiley \& Sons.

\bibitem[{Riemer et~al.(2019)Riemer, Cases, Ajemian, Liu, Rish, Tu, and
  Tesauro}]{riemer2018learning}
Riemer, M.; Cases, I.; Ajemian, R.; Liu, M.; Rish, I.; Tu, Y.; and Tesauro, G.
  2019.
\newblock Learning to Learn without Forgetting by Maximizing Transfer and
  Minimizing Interference.
\newblock In \emph{In International Conference on Learning Representations
  (ICLR)}.

\bibitem[{Sehnke et~al.(2008)Sehnke, Osendorfer, R{\"{u}}ckstie{\ss}, Graves,
  Peters, and Schmidhuber}]{sehnke2008parameter}
Sehnke, F.; Osendorfer, C.; R{\"{u}}ckstie{\ss}, T.; Graves, A.; Peters, J.;
  and Schmidhuber, J. 2008.
\newblock Policy Gradients with Parameter-Based Exploration for Control.
\newblock In Kurkov{\'{a}}, V.; Neruda, R.; and Koutn{\'{\i}}k, J., eds.,
  \emph{ICANN 2008}.

\bibitem[{Shelhamer et~al.(2017)Shelhamer, Mahmoudieh, Argus, and
  Darrell}]{Shelhamer2017dyn-pred}
Shelhamer, E.; Mahmoudieh, P.; Argus, M.; and Darrell, T. 2017.
\newblock Loss is its own Reward: Self-Supervision for Reinforcement Learning.
\newblock In \emph{{ICLR} 2017}.

\bibitem[{Silver, Yang, and Li(2013)}]{silver2013lifelong}
Silver, D.~L.; Yang, Q.; and Li, L. 2013.
\newblock Lifelong Machine Learning Systems: Beyond Learning Algorithms.
\newblock In \emph{Lifelong Machine Learning, Papers from the 2013 {AAAI}
  Spring Symposium, Palo Alto, California, USA, March 25-27, 2013}.

\bibitem[{Sutton and Barto(2018)}]{sutton2018reinforcement}
Sutton, R.~S.; and Barto, A.~G. 2018.
\newblock \emph{Reinforcement learning: An introduction}.
\newblock MIT press.

\bibitem[{Tieleman and Hinton(2012)}]{Tieleman2012}
Tieleman, T.; and Hinton, G. 2012.
\newblock {Lecture 6.5---RmsProp: Divide the gradient by a running average of
  its recent magnitude}.
\newblock COURSERA: Neural Networks for Machine Learning.

\bibitem[{Tirinzoni et~al.(2018)Tirinzoni, Sessa, Pirotta, and
  Restelli}]{pmlr-v80-tirinzoni18a}
Tirinzoni, A.; Sessa, A.; Pirotta, M.; and Restelli, M. 2018.
\newblock Importance Weighted Transfer of Samples in Reinforcement Learning.
\newblock In \emph{ICML 2018}.

\bibitem[{Vaswani et~al.(2017)Vaswani, Shazeer, Parmar, Uszkoreit, Jones,
  Gomez, Kaiser, and Polosukhin}]{vaswani2017attention}
Vaswani, A.; Shazeer, N.; Parmar, N.; Uszkoreit, J.; Jones, L.; Gomez, A.~N.;
  Kaiser, {\L}.; and Polosukhin, I. 2017.
\newblock Attention is all you need.
\newblock In \emph{NIPS 2017}, 5998--6008.

\bibitem[{Veach and Guibas(1995)}]{veach1995monte}
Veach, E.; and Guibas, L.~J. 1995.
\newblock Optimally combining sampling techniques for Monte Carlo rendering.
\newblock In Mair, S.~G.; and Cook, R., eds., \emph{SIGGRAPH 1995}.

\bibitem[{Williams(1992)}]{williams1992simple}
Williams, R.~J. 1992.
\newblock Simple statistical gradient-following algorithms for connectionist
  reinforcement learning.
\newblock \emph{Mach Learn}, 8(3): 229--256.

\bibitem[{Zhang, Satija, and Pineau(2018)}]{zhang2018decoupling}
Zhang, A.; Satija, H.; and Pineau, J. 2018.
\newblock Decoupling dynamics and reward for transfer learning.
\newblock In \emph{ICLR (workshop) 2018}.

\end{thebibliography}
}

\normalsize
\appendix
\onecolumn
\section{Proofs and Derivations}\label{apx:proofs}
In this appendix, we report the proofs of the results presented in the main paper.

\subsection{Gradient of the $\beta$-Step Ahead Expected Return}
\label{app:fut_ret_grad}

Using similar derivations as in PGT \cite{williams1992simple}, we derive the gradient of the $\beta$-step ahead expected return.  
\begin{align*}
    \nabla_{\vrho} \widehat{J}_{T,\alpha,\beta}(\nu_{\vrho}) 
    &= \nabla_{\vrho} \mathbb{E}^{\vrho}_{T,\alpha} \left[ \sum_{t=T-\alpha+1}^{T} r_t \omega^{T-t} \frac{\sum_{s=T+1}^{T+\beta} \nu_{\vrho}(\vtheta|s)}{\sum_{k=T-\alpha+1}^{T}\omega^{T-k} \nu_{\vrho} (\vtheta_t \vert k)}\right] 
    \\
    &= \sum_{t=T-\alpha+1}^{T} \nabla_{\vrho} \int_{\Theta}  \mathbb{E}_{t}^{\pi_{\vtheta}}[r] \omega^{T-t} \frac{\sum_{s=T+1}^{T+\beta} \nu_{\vrho}(\vtheta|s)}{\sum_{k=T-\alpha+1}^{T}\omega^{T-k}\nu_{\vrho} (\vtheta \vert k)}  \nu_{\vrho} (\vtheta \vert t)d\vtheta
    \\
    &=  \sum_{t=T-\alpha+1}^{T} \int_{\Theta}   \mathbb{E}_{t}^{\pi_{\vtheta}}[r] \omega^{T-t} \frac{\sum_{s=T+1}^{T+\beta} \nu_{\vrho}(\vtheta|s)}{\sum_{k=T-\alpha+1}^{T}\nu_{\vrho} (\vtheta \vert k)}  \nu_{\vrho} (\vtheta \vert t) \nabla_{\vrho}  \log\left( \frac{\sum_{s=T+1}^{T+\beta} \nu_{\vrho}(\vtheta_t|s)}{\sum_{k=T-\alpha+1}^{T}\omega^{T-k}\nu_{\vrho} (\vtheta \vert k)}  \nu_{\vrho} (\vtheta \vert t) \right) d\vtheta
    \\ 
    &= \mathbb{E}^{\vrho}_{T,\alpha}\left[  \sum_{t=T-\alpha+1}^{T} r_t \omega^{T-t} \frac{\sum_{s=T+1}^{T+\beta} \nu_{\vrho}(\vtheta_t|s)}{\sum_{k=T-\alpha+1}^{T}\omega^{T-k}\nu_{\vrho} (\vtheta_t \vert k)}  \nabla_{\vrho} \log \left( \frac{\sum_{s=T+1}^{T+\beta} \nu_{\vrho}(\vtheta_t|s)}{\sum_{k=T-\alpha+1}^{T}\omega^{T-k}\nu_{\vrho} (\vtheta_t \vert k)}  \nu_{\vrho} (\vtheta_t \vert t) \right)\right]
\end{align*}

\subsection{Proof of \cref{pp:var_bound}}

\ppvarbound*
\begin{proof}
We use $\widehat{\gamma}^t=\gamma^{t-T-1}$, $\widecheck{\gamma}^t= \gamma^{t-T+\alpha-1}$ and start with the following decomposition:
\begin{align}
    \mathbb{V}\mathrm{ar}^{\vrho}_{T,\alpha} \left[\overline{J}_{T,\alpha,\beta}\right] 
    &\leq 2\mathbb{V}\mathrm{ar}^{\vrho}_{T,\alpha} \left[\widecheck{J}_{T,\alpha}\right] + 2\mathbb{V}\mathrm{ar}^{\vrho}_{T,\alpha} \left[\widehat{J}_{T,\alpha,\beta}\right] \nonumber,
\end{align}
where we exploited the inequality $\Var[X+Y] \le 2\Var[X]+2\Var[Y]$ between arbitrary random variables. Now we deal with the two terms separately. For the first one:
\begin{align*}
    \mathbb{V}\mathrm{ar}^{\vrho}_{T,\alpha} \left[\widecheck{J}_{T,\alpha}\right] &= \mathbb{V}\mathrm{ar}^{\vrho}_{T,\alpha} \left[ \frac{1}{C_\omega} \sum_{t=T-\alpha+1}^{T} \omega^{T-t} \widecheck{\gamma}^t r_t \right] \\
    & \le \mathbb{E}^{\vrho}_{T,\alpha} \left[ \left(\frac{1}{C_\omega} \sum_{t=T-\alpha+1}^{T} \omega^{T-t} \widecheck{\gamma}^t r_t\right)^2 \right] \\
    & \le \|R\|_{\infty}^2  \left(\sum_{t=T-\alpha+1}^{T} \frac{\omega^{T-t}}{C_\omega} \widecheck{\gamma}^t\right)^2 \\
    &\le \|R\|_{\infty}^2  \left(\sum_{t=T-\alpha+1}^{T}  \widecheck{\gamma}^t \right)^2 = \left( \underbrace{\frac{1-\gamma^\alpha}{1-\gamma}}_{C_\gamma(\alpha)}\right)^2,
\end{align*}
having exploited $\frac{\omega^{T-t}}{C_\omega} \le 1$.
Let us now move to the second term:
\begin{align}
    \mathbb{V}\mathrm{ar}^{\vrho}_{T,\alpha} \left[\widehat{J}_{T,\alpha,\beta}\right] & = \mathbb{V}\mathrm{ar}^{\vrho}_{T,\alpha} \left[
        \sum_{t=T-\alpha+1}^{T} \omega^{T-t}  r_t  \frac{\sum_{s=T+1}^{T+\beta} \widehat{\gamma}^s\nu_{\vrho}(\vtheta_t\vert s)}{\sum_{k=T-\alpha+1}^T \omega^{T-t}\nu_{\vrho}(\vtheta_t \vert k)} 
    \right]\notag \\
    & \le \|R\|_\infty^2 \mathbb{E}^{\vrho}_{T,\alpha} \left[\left( \frac{1}{C_\omega}
        \sum_{t=T-\alpha+1}^{T} \omega^{T-t}  \frac{\sum_{s=T+1}^{T+\beta} \widehat{\gamma}^s\nu_{\vrho}(\vtheta_t\vert s)}{\frac{1}{C_\omega}\sum_{k=T-\alpha+1}^T \omega^{T-t}\nu_{\vrho}(\vtheta_t \vert k)}\right)^2 
    \right] \notag\\
    & \le \|R\|_\infty^2 \mathbb{E}^{\vrho}_{T,\alpha} \left[\frac{1}{C_\omega}
        \sum_{t=T-\alpha+1}^{T} \omega^{T-t} \left( \frac{\sum_{s=T+1}^{T+\beta} \widehat{\gamma}^s\nu_{\vrho}(\vtheta_t\vert s)}{\frac{1}{C_\omega} \sum_{k=T-\alpha+1}^T \omega^{T-t}\nu_{\vrho}(\vtheta_t \vert k)}\right)^2 
    \right] \label{p:0001} \\
    & = C_\gamma(\beta)^2 \|R\|_\infty^2 \mathbb{E}^{\vrho}_{T,\alpha} \left[\frac{1}{C_\omega}
        \sum_{t=T-\alpha+1}^{T} \omega^{T-t} \left( \underbrace{\frac{\frac{1}{C_\gamma(\beta)} \sum_{s=T+1}^{T+\beta} \widehat{\gamma}^s\nu_{\vrho}(\vtheta_t\vert s)}{\frac{1}{C_\omega} \sum_{k=T-\alpha+1}^T \omega^{T-t}\nu_{\vrho}(\vtheta_t \vert k)}}_{\ell(\vtheta_t)}\right)^2, \notag
    \right],
\end{align}
where in line~\eqref{p:0001} we applied Cauchy-Schwarz inequality to the summation. Let us now consider the expectation $\mathbb{E}^{\vrho}_{T,\alpha}$ that is computed under the distribution:
\begin{align*}
    D_0(s_0) \prod_{l=1}^{T} P_t(s_{l+1}|s_l, \pi_{\vtheta_l}(s_l)) \nu_{\vrho}(\vtheta_l|l).
\end{align*}
Now, consider an individual term at time $t$ of the summation:
\begin{align*}
   \mathbb{E}^{\vrho}_{T,\alpha} \left[\ell(\vtheta_t)^2 \right] &= \int_{s_0}\int_{\vtheta_0} \dots \int_{s_T} \int_{\vtheta_T}  D_0(s_0) \prod_{l=1}^{T} P_t(s_{l+1}|s_l, \pi_{\vtheta_l}(s_l)) \nu_{\vrho}(\vtheta_l|l) \ell(\vtheta_t)^2 \de s_0 \de \vtheta_0 \dots \de s_T \de \vtheta_T  \\
   & =  \int_{\vtheta_t} \nu_{\vrho}(\vtheta_t|t) \ell(\vtheta_t)^2 \de \vtheta_t, 
\end{align*}
where the last line is obtained by observing that all integrals cancel out apart for the one regarding $\vtheta_t$. Thus, we have:
\begin{align}
     \mathbb{E}^{\vrho}_{T,\alpha} \left[\frac{1}{C_\omega}
        \sum_{t=T-\alpha+1}^{T} \omega^{T-t} \ell(\vtheta_t)^2
    \right] & = \frac{1}{C_\omega}\sum_{t=T-\alpha+1}^{T} \omega^{T-t}  \int_{\vtheta_t} \nu_{\vrho}(\vtheta_t|t) \ell(\vtheta_t)^2 \de \vtheta_t \notag \\
    & =    \int_{\vtheta} \frac{1}{C_\omega}\sum_{t=T-\alpha+1}^{T} \omega^{T-t} \nu_{\vrho}(\vtheta|t) \ell(\vtheta)^2 \de \vtheta \label{p:0002} \\
    & = \int_{\vtheta} \frac{1}{C_\omega}\sum_{t=T-\alpha+1}^{T} \omega^{T-t} \nu_{\vrho}(\vtheta|t) \left( \frac{ \frac{1}{C_\gamma(\beta)} \sum_{s=T+1}^{T+\beta} \widehat{\gamma}^s\nu_{\vrho}(\vtheta\vert s)}{\frac{1}{C_\omega} \sum_{k=T-\alpha+1}^T \omega^{T-t}\nu_{\vrho}(\vtheta \vert k)}\right)^2 \de \vtheta,
\end{align}
where in line~\eqref{p:0002} we exploited the fact that $\vtheta_t$ is a dummy variable. The result is obtained by observing that the last expression corresponds to the exponentiated 2-\Renyi divergence.

\end{proof}

\subsection{Proof of \cref{pp:divergence_bound}}

\pdivbound*
\begin{proof}
    We apply \cref{pp:var_bound_2_steps_psi_first} to our bound with $\zeta_i=\frac{\widehat{\gamma}^i}{C_\gamma(\beta)}$, $\mu_j = \frac{\omega^{T-j}}{C_\omega}$, $\alpha=2$ and change the summation from $i\in[\![1,\dots,L]\!]$ to $i\in[\![T+1,\dots,T+\beta]\!]$ and $j\in[\![1,\dots,K]\!]$ to $j\in[\![T-\alpha+1,\dots,T]\!]$.
\end{proof}

\subsection{Proof of \cref{th:rl_bound}}

\thrlbound*
\begin{proof}
    As in \cite{metelli2018policy}, we use a Cantelli's inequality on the random variable $\overline{J}_{T,\alpha,\beta}(\vrho)$,
    \begin{align*}
        \mathbb{P}\left( \overline{J}_{T,\alpha,\beta}(\vrho) - \E[\overline{J}_{T,\alpha,\beta}(\vrho)]\geq \lambda\right) &\leq \frac{1}{1+\frac{\lambda^2}{\Var[\overline{J}_{T,\alpha,\beta}(\vrho)]}}.
    \end{align*}  
    we define $\delta = \frac{1}{1+\frac{\lambda^2}{\Var[\overline{J}_{T,\alpha,\beta}(\vrho)]}}$ and consider the complementary event, which yiels that with probability at least $1-\delta$,
    \begin{align*}
        E[\overline{J}_{T,\alpha,\beta}(\vrho)]\geq  \overline{J}_{T,\alpha,\beta}(\vrho) - 
        \sqrt{
            \frac{1-\delta}{\delta}  \mathbb{V}\mathrm{ar}^{\vrho}_T \left[\overline{J}_{T,\alpha,\beta}(\vrho)\right] 
        }.
    \end{align*}
    We now replace with the bound of the variance from \cref{pp:var_bound} and use the variational bound for the mixture of \Renyi-divergences from \cref{pp:divergence_bound} to obtain
    \begin{align*}
        \E[\overline{J}_{T,\alpha,\beta}(\vrho)]
        &\geq  \overline{J}_{T,\alpha,\beta}(\vrho) - \sqrt{
            \frac{1-\delta}{\delta} 2 \lVert R \rVert_{\infty}^2  
        \left(
            C_\gamma(\alpha)^2
            +
            C_\gamma(\beta)^2  
            d_2 \left(\frac{1}{C_\gamma(\beta)} \sum_{s=T+1}^{T+\beta} \widehat{\gamma}^s \nu_{\vrho}(\cdot|s) \left\Vert \frac{1}{C_\omega} \sum_{t=T-\alpha+1}^{T} \omega^{T-t} \nu_{\vrho}(\cdot|t)  \right.\right)
        \right)
        }
        \\
        &= \overline{J}_{T,\alpha,\beta}(\vrho) - \sqrt{
            \frac{1-\delta}{\delta} 2  \lVert R \rVert_{\infty}^2  
        \left(
            C_\gamma(\alpha)^2
            +
            C_\omega   
            \left( \sum_{ s=T+1}^{ T+\beta} \widehat{\gamma}^{s}\frac{1}{\left(\sum_{ k=T-\alpha+1}^{ T}\frac{\omega^{T-k}}{d_{2}(\nu_{\vrho}(\cdot\vert s)\left\Vert \nu_{\vrho}(\cdot\vert k)\right.)}\right)^{\frac{1}{2}}}\right)^2
        \right).
        }
    \end{align*}
    The result is obtained by recalling the definition of $B_{T,\alpha,\beta}(\vrho)$.
\end{proof}

\subsection{Bias Analysis and Proof of \cref{lem:bias_bound}}
\label{apx:bias}

We first derive a first result involving a tighter bound than the one provided in \cref{lem:bias_bound}, but with a more intricate expression. The bound also holds for $\omega=1$.

\begin{restatable}{lemma}{biasGeneralBound}
\label{lem:bias_general}
Under Assumptions~\ref{ass:sce} and~\ref{ass:sch}, the bias of the estimator $\widehat{J}_{T,\alpha,\beta}(\vrho)$, for $0<\omega\leq1$, can be bounded as:
\begin{align*}
    \left|J_{T,\beta}(\vrho) - \mathbb{E}^{\vrho}_{T,\alpha}[\widehat{J}_{T,\alpha,\beta}] \right|\le\left(L_{\mathcal{M}} + 2R_{\max}L_\nu\right) C_{\gamma}(\beta) \left(\omega 
    \frac{1-\alpha\omega^{\alpha-1} + (\alpha-1) \omega^\alpha}{(1-\omega)(1-\omega^\alpha)} + \frac{1}{1-\gamma}\right),
\end{align*}
where $C_\gamma(\xi)=\frac{1-\gamma^\xi}{1-\gamma}$ if $\gamma<1$ otherwise $C_\gamma(\xi)=\xi$ for $\xi \ge 1$.
In particular, when $\omega=1$, the bound is the limit at $\omega \rightarrow 1$ of the previous expression and reads
\begin{align*}
    \left|J_{T,\beta}(\vrho) - \mathbb{E}^{\vrho}_{T,\alpha}[\widehat{J}_{T,\alpha,\beta}] \right|\le \left(L_{\mathcal{M}} + 2R_{\max}L_\nu\right) C_{\gamma}(\beta) \left(\frac{\alpha-1}{2} + \frac{1}{1-\gamma}  \right) 
\end{align*}
\end{restatable}

\begin{proof}
Let us express explicitly the estimator:
\begin{align*}
    \mathbb{E}^{\vrho}_{T,\alpha}\left[\widehat{J}_{T,\alpha,\beta}(\vrho)\right] & = \sum_{t=T-\alpha+1}^T \omega^{T-t}  \int_{\Theta} \nu_{\vrho}(\vtheta|t) \frac{\sum_{s=T+1}^{T+\beta} \widehat{\gamma}^s \nu_{\vrho}(\vtheta|s)}{\sum_{k=T-\alpha+1}^T \omega^{T-t} \nu_{\vrho}(\vtheta|k)} \mathbb{E}_t^{\pi_{\vtheta}}[r]  \de \vtheta \\
    & = \sum_{s=T+1}^{T+\beta} \widehat{\gamma}^s \int_{\Theta}  \nu_{\vrho}(\vtheta|s) \frac{\sum_{t=T-\alpha+1}^T \omega^{T-t} \nu_{\vrho}(\vtheta|t) \mathbb{E}_t^{\pi_{\vtheta}}[r]  }{\sum_{k=T-\alpha+1}^T \omega^{T-k} \nu_{\vrho}(\vtheta|k)} \de \vtheta.
\end{align*}
Let us now observe that $J_{T,\beta}(\vrho) = \sum_{s=T+1}^{T+\beta} \widehat{\gamma}^s \int_{\Theta} \nu_{\vrho}(\vtheta|s) \mathbb{E}_s^{\pi_{\vtheta}}[r] \de \vtheta$. Thus, we have:
\begin{align*}
    \left\vert \mathbb{E}^{\vrho}_{T,\alpha}\left[\widehat{J}_{T,\alpha,\beta}(\vrho)\right] - J_{T,\beta}(\vrho)\right\vert & =\left| \sum_{s=T+1}^{T+\beta} \widehat{\gamma}^s \int_{\Theta}  \nu_{\vrho}(\vtheta|s) \frac{\sum_{t=T-\alpha+1}^T \omega^{T-t} \nu_{\vrho}(\vtheta|t) \mathbb{E}_t^{\pi_{\vtheta}}[r]  }{\sum_{k=T-\alpha+1}^T \omega^{T-k} \nu_{\vrho}(\vtheta|k)} \de \vtheta - \sum_{s=T+1}^{T+\beta} \widehat{\gamma}^s \int_{\Theta} \nu_{\vrho}(\vtheta|s) \mathbb{E}_s^{\pi_{\vtheta}}[r] \de \vtheta \right| \\
    & = \left| \sum_{s=T+1}^{T+\beta} \widehat{\gamma}^s \int_{\Theta}  \nu_{\vrho}(\vtheta|s) \frac{\sum_{t=T-\alpha+1}^T \omega^{T-t} \nu_{\vrho}(\vtheta|t) \left(\mathbb{E}_t^{\pi_{\vtheta}}[r] - \mathbb{E}_s^{\pi_{\vtheta}}[r] \right) }{\sum_{k=T-\alpha+1}^T \omega^{T-k} \nu_{\vrho}(\vtheta|k)} \de \vtheta  \right|.
\end{align*}
Now we proceed as follows, by renaming $\overline{\nu}(\vtheta) = \frac{1}{C_{\omega}} \sum_{k=T-\alpha+1}^T \omega^{T-k} \nu_{\vrho}(\vtheta|k)$:
\begin{align*}
   \Bigg| \sum_{s=T+1}^{T+\beta} \widehat{\gamma}^s & \int_{\Theta}  \nu_{\vrho}(\vtheta|s) \frac{\sum_{t=T-\alpha+1}^T \omega^{T-t} \nu_{\vrho}(\vtheta|t) \left(\mathbb{E}_t^{\pi_{\vtheta}}[r] - \mathbb{E}_s^{\pi_{\vtheta}}[r] \right) }{\sum_{k=T-\alpha+1}^T \omega^{T-k} \nu_{\vrho}(\vtheta|k)} \de \vtheta  \Bigg|  \\
   & = \left| \sum_{s=T+1}^{T+\beta} \widehat{\gamma}^s \int_{\Theta}  \left(\nu_{\vrho}(\vtheta|s) \pm \overline{\nu}_{\vrho}(\vtheta)\right) \frac{\sum_{t=T-\alpha+1}^T \omega^{T-t} \nu_{\vrho}(\vtheta|t)  \left(\mathbb{E}_t^{\pi_{\vtheta}}[r] - \mathbb{E}_s^{\pi_{\vtheta}}[r] \right) }{C_\omega \overline{\nu}(\vrho)} \de \vtheta  \right|\\
    & \le \underbrace{\left| \sum_{s=T+1}^{T+\beta} \widehat{\gamma}^s \int_{\Theta}  \overline{\nu}_{\vrho}(\vtheta)  \frac{\sum_{t=T-\alpha+1}^T \omega^{T-t} {\nu}_{\vrho}(\vtheta|t) \left(\mathbb{E}_t^{\pi_{\vtheta}}[r] - \mathbb{E}_s^{\pi_{\vtheta}}[r] \right) }{C_\omega \overline{\nu}(\vrho)} \de \vtheta  \right|}_{\text{(a)}} \\
    & \quad + \underbrace{\left| \sum_{s=T+1}^{T+\beta} \widehat{\gamma}^s \int_{\Theta}  \left(\nu_{\vrho}(\vtheta|s) - \overline{\nu}_{\vrho}(\vtheta)\right) \frac{\sum_{t=T-\alpha+1}^T \omega^{T-t} \nu_{\vrho}(\vtheta|t) \left(\mathbb{E}_t^{\pi_{\vtheta}}[r] - \mathbb{E}_s^{\pi_{\vtheta}}[r] \right) }{C_\omega \overline{\nu}(\vrho)} \de \vtheta  \right|}_{\text{(b)}}.
\end{align*}
We consider the two terms separately. Let us start from (a):
\begin{align*}
    \text{(a)} & = \frac{1}{C_\omega} \left| \sum_{s=T+1}^{T+\beta} \widehat{\gamma}^s \int_{\Theta}  \sum_{t=T-\alpha+1}^T \omega^{T-t} {\nu}_{\vrho}(\vtheta|t) \left(\mathbb{E}_t^{\pi_{\vtheta}}[r] - \mathbb{E}_s^{\pi_{\vtheta}}[r] \right) \de \vtheta  \right| \\
    & \le \frac{1}{C_\omega}  \sum_{s=T+1}^{T+\beta} \widehat{\gamma}^s \int_{\Theta}  \sum_{t=T-\alpha+1}^T \omega^{T-t} {\nu}_{\vrho}(\vtheta|t) \left|\mathbb{E}_t^{\pi_{\vtheta}}[r] - \mathbb{E}_s^{\pi_{\vtheta}}[r] \right| \de \vtheta  \\
    &  \le \frac{L_{\mathcal{M}}}{C_\omega}  \sum_{s=T+1}^{T+\beta} \widehat{\gamma}^s  \sum_{t=T-\alpha+1}^T \omega^{T-t} \int_{\Theta} \nu_{\vrho}(\vtheta|t) \left|t-s \right|  \de \vtheta \\
    & \le \frac{L_{\mathcal{M}}}{C_\omega}  {\sum_{s=T+1}^{T+\beta} \widehat{\gamma}^s   \sum_{t=T-\alpha+1}^T \omega^{T-t} \left|t-s \right|}, 
\end{align*}
where we employed $\int_{\Theta} {\nu}_{\vrho}(\vtheta|t) \de \vtheta = 1$ in the last passage and Assumption~\ref{ass:sce} in the last passage but one. Let us now move to (b):
\begin{align*}
    (b) & \le 2R_{\max} \sum_{s=T+1}^{T+\beta} \widehat{\gamma}^s \int_{\Theta}  \left|\nu_{\vrho}(\vtheta|s) - \overline{\nu}_{\vrho}(\vtheta)\right| \frac{\sum_{t=T-\alpha+1}^T \omega^{T-t} \nu_{\vrho}(\vtheta|t) }{C_\omega \overline{\nu}(\vrho)} \de \vtheta \\
    & = 2R_{\max} \sum_{s=T+1}^{T+\beta} \widehat{\gamma}^s \int_{\Theta}  \left|\nu_{\vrho}(\vtheta|s) - \overline{\nu}_{\vrho}(\vtheta)\right|  \de \vtheta \\
    & \le 2R_{\max} \sum_{s=T+1}^{T+\beta} \widehat{\gamma}^s  \frac{1}{C_{\omega}} \sum_{t=T-\alpha+1}^{T} \omega^{T-t} \int_{\Theta} \left|\nu_{\vrho}(\vtheta|s) - \nu_{\vrho}(\vtheta|t) \right|  \de \vtheta \\
    & \le  \frac{2R_{\max} L_{\nu}}{C_{\omega}} \sum_{s=T+1}^{T+\beta} \widehat{\gamma}^s   \sum_{t=T-\alpha+1}^{T} \omega^{T-t} \left|t - s \right|,  
\end{align*}
where we used Assumption~\ref{ass:sch} in the last passage.

We now use a similar derivation as in Lemma 3.4 of \cite{jagerman2019people}. Observe that, setting $m = s-T$ and $n=T-t$,
\begin{align}
    \frac{1}{C_\omega}\sum_{t=T-\alpha+1}^{T} \omega^{T-t} (s-t)
    =  \frac{1}{C_\omega}\sum_{n=0}^{\alpha-1} \omega^{n} (m+n) = m + \frac{1}{C_\omega}
    \sum_{n=1}^{\alpha-1} \omega^{n} n.\label{eq:sum_m_n}
\end{align}
If $\omega<1$, we have:
\begin{align*}
    \frac{1}{C_\omega}\sum_{t=T-\alpha+1}^{T} \omega^{T-t} (s-t)
    &= m + \frac{1}{C_\omega} \omega \frac{d}{d\omega}
    \sum_{n=1}^{\alpha-1} \omega^{n} 
    \\
    &= m + \frac{1}{C_\omega} \omega 
    \frac{1-\alpha\omega^{\alpha-1} + (\alpha-1) \omega^\alpha}{(1-\omega)^2}
    \\
    &= m + \omega 
    \frac{1-\alpha\omega^{\alpha-1} + (\alpha-1) \omega^\alpha}{(1-\omega)(1-\omega^\alpha)},
\end{align*}

which yields

\begin{align*}
    \left|J_{T,\beta}(\vrho) - \mathbb{E}^{\vrho}_{T,\alpha}\left[\widehat{J}_{T,\alpha,\beta}\right] \right| 
    &\le \frac{L_{\mathcal{M}} + 2R_{\max}L_\nu}{C_\omega} \sum_{s=T+1}^{T+\beta} \sum_{t=T-\alpha+1}^{T} \widehat{\gamma}^s    \omega^{T-t} (s-t) 
    \\
    &= \left(L_{\mathcal{M}} + 2R_{\max}L_\nu\right) \sum_{s=T+1}^{T+\beta} \widehat{\gamma}^s \left(m + \omega 
    \frac{1-\alpha\omega^{\alpha-1} + (\alpha-1) \omega^\alpha}{(1-\omega)(1-\omega^\alpha)}\right)
    \\
    &= \left(L_{\mathcal{M}} + 2R_{\max}L_\nu\right) \sum_{s=T+1}^{T+\beta} \widehat{\gamma}^s \left((s-T) + \omega 
    \frac{1-\alpha\omega^{\alpha-1} + (\alpha-1) \omega^\alpha}{(1-\omega)(1-\omega^\alpha)}\right)
    \\
    &= \left(L_{\mathcal{M}} + 2R_{\max}L_\nu\right)  \left(\frac{1-\gamma^\beta}{1-\gamma}  \omega 
    \frac{1-\alpha\omega^{\alpha-1} + (\alpha-1) \omega^\alpha}{(1-\omega)(1-\omega^\alpha)} + \sum_{k=0}^{\beta-1} \gamma^k (k+1)\right)
    \\
    &= \left(L_{\mathcal{M}} + 2R_{\max}L_\nu\right)  \left(\frac{1-\gamma^\beta}{1-\gamma}  \omega 
    \frac{1-\alpha\omega^{\alpha-1} + (\alpha-1) \omega^\alpha}{(1-\omega)(1-\omega^\alpha)} + \frac{1-\gamma^\beta}{(1-\gamma)^2}\right)
    \\
    &= \left(L_{\mathcal{M}} + 2R_{\max}L_\nu\right) \frac{1-\gamma^\beta}{1-\gamma} \left(\omega 
    \frac{1-\alpha\omega^{\alpha-1} + (\alpha-1) \omega^\alpha}{(1-\omega)(1-\omega^\alpha)} + \frac{1}{1-\gamma}\right).
\end{align*}

In the case $\omega=1$, the \cref{eq:sum_m_n} becomes:
\begin{align*}
    \frac{1}{C_\omega}\sum_{t=T-\alpha+1}^{T} \omega^{T-t} (s-t)
     = m + \frac{1}{\alpha}
    \sum_{n=1}^{\alpha-1} n = m + \frac{\alpha-1}{2}.
\end{align*}
Thus, the bound becomes:
\begin{align*}
    \left|J_{T,\beta}(\vrho) - \mathbb{E}^{\vrho}_{T,\alpha}\left[\widehat{J}_{T,\alpha,\beta}\right] \right| 
    &\leq \left(L_{\mathcal{M}} + 2R_{\max}L_\nu\right) \sum_{s=T+1}^{T+\beta} \widehat{\gamma}^s \left(m +  \frac{\alpha-1}{2}\right)
    \\
    &= \left(L_{\mathcal{M}} + 2R_{\max}L_\nu\right) \left(C_{\gamma}(\beta)  \frac{\alpha-1}{2} + \sum_{k=0}^{\beta-1}\gamma^k (k+1)  \right) 
    \\
    &= \left(L_{\mathcal{M}} + 2R_{\max}L_\nu\right) C_{\gamma}(\beta) \left(\frac{\alpha-1}{2} + \frac{1}{1-\gamma}  \right). 
\end{align*}
\end{proof}

\biasBound*
\begin{proof}
We start from the bound is \cref{lem:bias_general} and observe that  $\frac{1-\alpha\omega^{\alpha-1} + (\alpha-1) \omega^\alpha}{(1-\omega)(1-\omega^\alpha)} 
    \leq \frac{1}{1-\omega} = C_\omega $ which yields the result.

\end{proof}

\section{On the Variational Bounds of \Renyi Divergence between Mixture Distributions}
\label{app:var_bound}
In this appendix, we discuss different approaches to obtain a useful bound on the \Renyi divergence between mixture distributions. Let $\Psi = \sum_{i=1}^L \zeta_i P_i$ and $\Phi = \sum_{j=1}^K \mu_j Q_j$ with $\forall i \in [\![1,L]\!], \zeta_i \in [0,1]$, $\forall j \in [\![1,K]\!], \mu_i \in [0,1]$, $\sum_{i=1}^L \zeta_i = 1$ and $\sum_{j=1}^K \mu_j = 1$ be two mixtures of probabilities. 
We are interested in finding an upper-bound of $d_\alpha(\Psi\left\Vert\Phi\right.)$ for $\alpha\geq1$.
We first recall a cornerstone results from \cite{papini2019optimistic}. 

\begin{lemma}[Lemma 4, \cite{papini2019optimistic}]
\label{pp:papini_lemma}
    Let $\{\psi_{ij}\}_{\substack{i\in[\![1,L]\!]\\j\in[\![1,K]\!]}}$ and $\{\phi_{ij}\}_{\substack{i\in[\![1,L]\!]\\j\in[\![1,K]\!]}}$ be two sets of variational parameters s.t. $\phi_{ij}\geq 0$, $\phi_{ij}\geq 0$, $\sum_{i=1}^L \psi_{ij} = \mu_j$ and $\sum_{j=1}^K \phi_{ij} = \zeta_i$. Then for any $\alpha\geq 1$, and for the previously defined mixture of probabilities it holds that:
    \begin{align}
        d_\alpha(\Psi\left\Vert\Phi\right.) \leq \sum_{i=1}^L\sum_{j=1}^K \phi_{ij}^{\alpha} \psi_{ij}^{1-\alpha} d_{\alpha}(P_i\left\Vert Q_j\right)^{\alpha-1}.\label{eq:papini_lemma}
    \end{align}
\end{lemma}

We wish to find the values of $\{\psi_{ij}\}$ and $\{\phi_{ij}\}$ which minimize the upper-bound as to make it tighter.
A straightforward yet unsuccessful idea is to take uniform values. We list six alternative approaches to finding a better bound for $d_\alpha(\Psi\left\Vert\Phi\right.)$. We then compare them on a toy problem and select the most promising one. 

\subsection{Direct Convex Optimization}
Optimizing for the best $\phi_{ij}\geq 0$ and $\psi_{ij}\geq 0$ under their constraints for a given pair of mixtures is a convex problem.

\paragraph{With Reset}
The first possibility is to start from a uniform distribution over the variational parameters then follow the gradient for a given number of steps and return the values of the variational parameters that will then be used in the bound. The process is repeated each time the variance bound is needed in the main algorithm. We refer to this approach as \textit{direct optimization with reset}.

\paragraph{Without Reset}
The idea is the same but the value of the variational parameters is not reset in between optimization, instead, the value of the variational parameters is kept in memory and reused as initialization each time the variance bound is needed. We refer to this approach as \textit{direct optimization without reset}.

\subsection{Two Steps Minimization}
Since solving an optimization problem whenever computing the bound might be inefficient, we propose in the following an alternative approach based on a bound. We again state a useful results from \cite{papini2019optimistic}.

\begin{thr}[Theorem 5, \citet{papini2019optimistic}]
\label{th:mix_phi}
    Let $P$ be a probability measure and consider the previous mixture $\Phi$, then for any $\alpha\geq1$, one has:
    \begin{align*}
        d_\alpha(P\left\Vert\Phi\right.) \leq     \frac{1}{\sum_{j=1}^K\frac{\mu_j}{d_{\alpha}( P \Vert Q_j)}}.
    \end{align*}
\end{thr}

Notice that the bound is the harmonic mean of the \Renyi-divergences between $P$ and the component of the mixture. We now derive a result for the case $d_\alpha(\Psi\left\Vert Q\right)$, where $\Psi$ is a mixture distribution.
\begin{prop}
\label{pp:mix_psi}
    Let $Q$ be a probability measure and consider the previous mixture $\Psi$, then for any $\alpha\geq1$, one has:
    \begin{align*}
        d_\alpha(\Psi\left\Vert Q\right.) \leq \left(\sum_{i=1}^L \zeta_i d_\alpha(P_i\Vert Q)^{\frac{\alpha-1}{\alpha}}\right)^{\frac{\alpha}{\alpha-1}}.
    \end{align*}
\end{prop}
\begin{proof}
Since the variational bound in Equation~\eqref{eq:papini_lemma} is convex in $\{\psi_i\}$, we can find the optimal value of $\{\psi_i\}$ via Lagrange multipliers following a similar reasoning as in \cite{papini2019optimistic}:
\begin{align*}
    \psi_i = \frac{\zeta_i d_\alpha(P_i \Vert Q)^{\frac{\alpha-1}{\alpha}}}{\sum_{i=1}^L \zeta_i d_\alpha(P_i\Vert Q)^{\frac{\alpha-1}{\alpha}}}.
\end{align*}
Then, we replace in the original problem:
\begin{align*}
    d_\alpha(\Psi\left\Vert Q\right.)^{\alpha-1} 
    &\leq \sum_{i=1}^L \zeta_i^\alpha \left(\frac{\zeta_i d_\alpha(P_i \Vert Q)^{\frac{\alpha-1}{\alpha}}}{\sum_{i=1}^L \zeta_i d_\alpha(P_l\Vert Q)^{\frac{\alpha-1}{\alpha}}} \right)^{1-\alpha} d_\alpha(P_i \Vert Q)^{\alpha-1}
    \\
    &= \sum_{i=1}^L \zeta_i\frac{ d_\alpha(P_i \Vert Q)^{(1-\alpha)\frac{\alpha-1}{\alpha} +\alpha-1}}{\left(\sum_{l=1}^L \zeta_l d_\alpha(P_l\Vert Q)^{\frac{\alpha-1}{\alpha}}\right)^{1-\alpha}} 
    \\
    &= \frac{\sum_{i=1}^L \zeta_i d_\alpha(P_i \Vert Q)^{\frac{\alpha-1}{\alpha} }}{\left(\sum_{i=1}^L \zeta_i d_\alpha(P_i\Vert Q)^{\frac{\alpha-1}{\alpha}}\right)^{1-\alpha}} 
    \\
    &= \left(\sum_{i=1}^L \zeta_i d_\alpha(P_i\Vert Q)^{\frac{\alpha-1}{\alpha}}\right)^{\alpha} .
\end{align*}
\end{proof}

Note that this is now the weighted power mean of exponent $\frac{\alpha-1}{\alpha}$. We can combine Theorem~\ref{th:mix_phi} and Proposition~\ref{pp:mix_psi} to find a bound for $d_\alpha(\Psi\left\Vert\Phi\right)$. There are 2 ways of doing so, presented in the following paragraphs. 

\subsection{Two Steps $\psi$ First}
We first use \cref{pp:mix_psi} then \cref{th:mix_phi}. We call this approach \textit{two steps $\psi$ first}. Doing so, we have:
\begin{prop}
\label{pp:var_bound_2_steps_psi_first}
    Under the same assumptions of \cref{pp:papini_lemma}, one has:
    \begin{align*}
        d_\alpha(\Psi\left\Vert \Phi\right.) 
        &\leq 
        \left(\sum_{i=1}^L \zeta_i d_\alpha(P_i\Vert \Psi)^{\frac{\alpha-1}{\alpha}}\right)^{\frac{\alpha}{\alpha-1}}\\
        &\leq 
        \left(\sum_{i=1}^L \zeta_i \frac{1}{\left(\sum_{j=1}^K\frac{\mu_j}{d_{\alpha}( P_i \Vert Q_j)}\right)^{\frac{\alpha-1}{\alpha}}}\right)^{\frac{\alpha}{\alpha-1}}.
    \end{align*}
\end{prop}
We can now apply this results to the bound in \cref{pp:var_bound} with $\alpha=2$ and it yields the result from \cref{th:rl_bound}.

\subsection{Two Steps $\phi$ First}
Alternatively, we can use \cref{pp:mix_psi} first and then \cref{th:mix_phi}. We call this approach \textit{two steps $\phi$ first}. This yields:
\begin{prop}
    Under the same assumptions of \cref{pp:papini_lemma},
    \begin{align*}
    d_\alpha(\Psi\left\Vert \Phi\right.)
    &\leq \frac{1}{\sum_{j=1}^K \frac{\mu_j}{d_\alpha(\Psi\Vert Q_j)}}
    \\
    &\leq \frac{1}{\sum_{j=1}^K \mu_j \left(\sum_{i=1}^L \zeta_i d_\alpha(P_i\Vert Q_j)^{\frac{\alpha-1}{\alpha}}\right)^{-\frac{\alpha}{\alpha-1}} }.
\end{align*}
\end{prop}

We now apply this results to the bound in \cref{pp:var_bound} to yield the following result.

\begin{prop}[Lower bound with \textit{two steps $\phi$ first}]
    For $\delta>0$, with probability at least $1-\delta$, it holds that 
\begin{align*}
    \mathbb{E}\left[\overline{J}_{T,\alpha,\beta} \right] 
    &\geq 
    \overline{J}_{T,\alpha,\beta} - \sqrt{
            \frac{1-\delta}{\delta} 2\lVert R \rVert_{\infty}^2 \left( C_\gamma(\alpha)^2 +
            C_\omega
            \frac{1}{\sum_{ k=T-\alpha+1}^{ T}\omega^{T-k} \left(\sum_{ s=T+1}^{ T+\beta} \widehat{\gamma}^s d_2(\nu_{\vrho}(\cdot\vert s)\Vert \nu_{\vrho}(\cdot\vert k))^{\frac{1}{2}}\right)^{-2} }.
        \right)
        }.
\end{align*}
\end{prop}

\subsection{One Step then Uniform}
Another possible approach is to find the optimal value of the variational parameters $\psi_{ij}$ as a function of the parameters $\phi_{ij}$ and then replace it in \cref{eq:papini_lemma}. It is also possible to do the other way around. The result is given in the next proposition.
\begin{prop}
    Under the same conditions as in \cref{pp:papini_lemma}. The optimal values of $\psi_{ij}$ are
    \begin{align*}
        \psi_{ij} = \mu_{j} \frac{\phi_{ij}d_{\alpha}(P_i\Vert Q_j)^{\frac{\alpha-1}{\alpha}}}{\sum_{l=1}^L \phi_{lj}d_{\alpha}(P_l\Vert Q_j)^{\frac{\alpha-1}{\alpha}}},
    \end{align*}
    and the optimal values of $\phi_{ij}$ are
    \begin{align*}
        \phi_{ij} = \frac{\zeta_i \psi_{ij}}{ d_\alpha(P_i\Vert Q_j)} \left(\sum_{k=1}^K\frac{\psi_{ik}}{d_\alpha(P_i\Vert Q_k)}\right)^{-1}.
    \end{align*}
\end{prop}
\begin{proof}
    By \cref{eq:papini_lemma}, one knows that
    \begin{align*}
        d_\alpha(\Psi\Vert\Phi)^{\alpha-1}\leq\sum_{i=1}^{L}\sum_{j=1}^K \phi_{ij}^\alpha \psi_{ij}^{1-\alpha} d_\alpha(P_i\Vert Q_j)^{\alpha-1}.
    \end{align*}
    We write the Lagrangian of the optimisation problem:
    \begin{align*}
        \mathcal{L}(\phi_{ij},\phi_{ij},\lambda_i^{\zeta},\lambda_j^{\mu})=\sum_{i=1}^{L}\sum_{j=1}^K \phi_{ij}^\alpha \psi_{ij}^{1-\alpha} d_\alpha(P_i\Vert Q_j)^{\alpha-1} - \sum_{i=1}^L \lambda_i^{\zeta} (\sum_{j=1}^K \phi_{ij} - \zeta_i)
        - \sum_{j=1}^K \lambda_j^{\mu} (\sum_{i=1}^L \psi_{ij} - \mu_j).
    \end{align*}
    We now look at the zero of the derivative with respect to the variational variables:
    \begin{align*}
        \frac{\partial \mathcal{L}}{\partial \phi_{ij}} &= \alpha \phi_{ij}^{\alpha-1} \psi_{ij}^{1-\alpha} d_\alpha (P_i\Vert Q_j)^{\alpha-1}-\lambda_i^{\zeta} =0.
    \end{align*}
    It implies that: 
    \begin{align*}
        \phi_{ij} = \frac{{\lambda_i^{\zeta}}^{\frac{1}{\alpha-1}}\psi_{ij}}{\alpha^{\frac{1}{\alpha-1}} d_\alpha(P_i\Vert Q_j)}.
    \end{align*}
    Recall that $\sum_{j}\phi_{ij}=\zeta_i$ so:
    \begin{align*}
        \left(\frac{{\lambda_i^{\zeta}}}{\alpha}\right)^{\frac{1}{\alpha-1}} \sum_{k=1}^K \frac{\psi_{ik}}{d_\alpha(P_i\Vert Q_k)} = \zeta_l.
    \end{align*}
    This gives the value of $\lambda_i^{\zeta}$:
    \begin{align*}
        \lambda_i^{\zeta} = \frac{\alpha \zeta_i^{\alpha-1}}{\left(\sum_{k=1}^K\frac{\psi_ik}{d_\alpha(P_i\Vert Q_k)}\right)^{\alpha-1}}.
    \end{align*}
    Finally, by replacing,
    \begin{align}
    \label{eq:phi_optim}
        \phi_{ij} = \frac{\zeta_i \psi_{ij}}{ d_\alpha(P_i\Vert Q_j)} \left(\sum_{k=1}^K\frac{\psi_{ik}}{d_\alpha(P_i\Vert Q_k)}\right)^{-1}.
    \end{align}
    We do the same for $\psi_{ij}$:
    \begin{align*}
        \frac{\partial \mathcal{L}}{\partial \psi_{ij}} &= (1-\alpha) \psi_{ij}^{\alpha} \psi_{ij}^{-\alpha} d_\alpha (P_i\Vert Q_j)^{\alpha-1}-\lambda_j^{\mu} = 0.
    \end{align*}
    It implies that 
    \begin{align*}
        \psi_{ij} = \left(\frac{1-\alpha}{\lambda_j^{\mu}}\right)^{\frac{1}{\alpha}}\phi_{ij}d_\alpha(P_i\Vert Q_j)^{\frac{\alpha-1}{\alpha}}.
    \end{align*}
    Recall that $\sum_{i}\psi_{ij}=\beta_j$ so:
    \begin{align*}
        \left(\frac{1-\alpha}{\lambda_j^{\mu}}\right)^{\frac{1}{\alpha}} \sum_{l=1}^L \phi_{lj}d_\alpha(P_l\Vert Q_j)^{\frac{\alpha-1}{\alpha}} = \mu_l.
    \end{align*}
    This gives the value of $\lambda_j^{\beta}$:
    \begin{align*}
        \lambda_j^{\beta} = \frac{1-\alpha}{\mu_j^\alpha}\left(\sum_{l=1}^L \phi_{lj}d_\alpha(P_l\Vert Q_j)^{\frac{\alpha-1}{\alpha}}\right)^\alpha .
    \end{align*}
    Finally, by replacing,
    \begin{align}
    \label{eq:psi_optim}
        \psi_{ij} = \mu_j \frac{\phi_{ij}d_\alpha(P_i\Vert Q_j)^{\frac{\alpha-1}{\alpha}}}{\sum_{l=1}^L \phi_{lj}d_\alpha(P_l\Vert Q_j)^{\frac{\alpha-1}{\alpha}}}.
    \end{align}
\end{proof}

\subsubsection{Uniform $\psi$}
    We use the optimal value of $\phi_{ij}$ from \cref{eq:phi_optim} inside \cref{eq:papini_lemma} and then use a uniform value of $\psi_{ij}$. We call this approach \textit{uniform $\psi$}. 
    Following this approach we have the following proposition. 
    \begin{prop}
    Under similar assumption as in \cref{pp:papini_lemma}, one has
        \begin{align*}
            d_{\alpha}(\Psi\vert\Phi)
            &\leq \sum_{i=1}^L \frac{\zeta_i^\alpha}{L^{1-\alpha}}
            \left(
                \sum_{j=1}^K \frac{\mu_j}{d_\alpha(P_i \Vert Q_j)}
            \right)^{1-\alpha}.
        \end{align*}
    \end{prop}
    \begin{proof}
        We inject the value of $\phi_{ij}$ from \cref{eq:phi_optim} inside \cref{eq:papini_lemma} to get
        \begin{align*}
            d_{\alpha}(\Psi\vert\Phi)
            &\leq \sum_{i=1}^L\sum_{j=1}^K \zeta_i^\alpha \frac{\psi_{ij}}{d_\alpha(P_i \Vert Q_j)}
            \left(
                \sum_{k=1}^K \frac{\psi_{ik}}{d_\alpha(P_i \Vert Q_k)}
            \right)^{-\alpha}
            \\
            &= \sum_{i=1}^L \zeta_i^\alpha
            \left(
                \sum_{j=1}^K \frac{\psi_{ij}}{d_\alpha(P_i \Vert Q_j)}
            \right)^{1-\alpha}.
        \end{align*}
        Then we wish now to find values of $\psi_{ij}$ to minimize this bound. Recall that the constraints are $\psi_{ij}\geq 0$ and $\sum_{i=1}^L \psi_{ij}=\mu_j$. We choose a uniform value, that is $\psi_{ij}=\frac{\mu_j}{L}$.
        This gives the result.
    \end{proof}
    
    We now apply this results to the bound in \cref{pp:var_bound} to yield the following result.
    
\begin{prop}[Lower bound with \textit{uniform $\psi$}]
    For $\delta>0$, with probability at least $1-\delta$, it holds that 
    \begin{align*}
        \mathbb{E}\left[\overline{J}_{T,\alpha,\beta} \right] 
        &\geq  \overline{J}_{T,\alpha,\beta} - 
        \sqrt{
            \frac{1-\delta}{\delta} 2  \lVert R \rVert_{\infty}^2  
        \left(
            C_\gamma(\alpha)^2 +
            \beta C_\omega 
            \sum_{s=T+1}^{T+\beta} \widehat{\gamma}^{2s}
            \left(
                \sum_{k=T-\alpha+1}^T\frac{\omega^{T-k}}{ d_2(\nu_{\vrho}(\cdot\vert s)\Vert \nu_{\vrho}(\cdot\vert k))}
            \right)^{-1}.
        \right)
        }
    \end{align*}
\end{prop}

\subsubsection{Uniform $\phi$}
We repeat the previous approach but switching the role of $\phi_{ij}$ and $\psi_{ij}$.
    We derive the following proposition. 
    \begin{prop}
    Under similar assumption as in \cref{pp:papini_lemma}, one has
        \begin{align*}
            d_{\alpha}(\Psi\vert\Phi)
            &\leq \sum_{j=1}^K \frac{\mu_j^{1-\alpha}}{K^{\alpha}}  \left(\sum_{i=1}^L \zeta_i d_\alpha(P_i \Vert Q_j)^{\frac{\alpha-1}{\alpha}}
            \right)^{\alpha}.
        \end{align*}
    \end{prop}
    \begin{proof}
        We inject the value of $\psi_{ij}$ from \cref{eq:phi_optim} inside \cref{eq:papini_lemma} to get
        \begin{align*}
            d_{\alpha}(\Psi\vert\Phi)
            &\leq \sum_{i=1}^L\sum_{j=1}^K \mu_j^{1-\alpha} \frac{\phi_{ij}d_\alpha(P_i \Vert Q_j)^{\frac{\alpha-1}{\alpha}}}{\left(\sum_{l=1}^L \phi_{lj}d_\alpha(P_l \Vert Q_j)^{\frac{\alpha-1}{\alpha}}\right)^{1-\alpha}}
            \\
            &= \sum_{j=1}^K \mu_j^{1-\alpha} \left(\sum_{i=1}^L \phi_{ij}d_\alpha(P_i \Vert Q_j)^{\frac{\alpha-1}{\alpha}}
            \right)^{\alpha}.
        \end{align*}
        Then we wish now to find values of $\phi_{ij}$ to minimize this bound. Recall that the constraints are $\phi_{ij}\geq 0$ and $\sum_{j=1}^K \phi_{ij}=\zeta_i$. We choose a uniform value, that is $\phi_{ij}=\frac{\zeta_i}{K}$.
        This gives the result.
    \end{proof}
    
    We now apply this results to the bound in \cref{pp:var_bound} to yield the following result.
    
\begin{prop}[Lower bound with \textit{uniform $\psi$}]
    For $\delta>0$, with probability at least $1-\delta$, it holds that 
    \begin{align*}
        \mathbb{E}\left[\overline{J}_{T,\alpha,\beta} \right] 
        &\geq  \overline{J}_{T,\alpha,\beta} - 
        \sqrt{
            \frac{1-\delta}{\delta} 2 C^2 \alpha \lVert R \rVert_{\infty}^2  
       \left(
            C_\gamma(\alpha)^2 +
            \frac{C_\omega}{\alpha^2}
            \sum_{k=T-\alpha+1}^T \frac{1}{\omega^{T-k}} \left(\sum_{s=T+1}^{T+\beta}\widehat{\gamma}^s d_2(\nu_{\vrho}(\cdot\vert s)\Vert \nu_{\vrho}(\cdot\vert k))^{\frac{1}{2}}
            \right)^{2}
        \right).
        }
    \end{align*}
\end{prop}

\subsection{Comparison of the Bounds}
In this section, we discuss, from the six lower-bounds of the variance of the total return $\overline{J}_{T,\alpha,\beta}(\vrho)$, which is tighter and, when used in the optimization, which constrains the hyper-policy toward being stationary. 

To do so, we design the following test.
Our hyper-policy will be sinusoidal, that is, it will output the following mean for the policy parameter $\theta$:
\begin{align*}
    \theta_t = A \sin(\phi t + \psi) + B.
\end{align*}
We choose such a hyper-policy since the scale parameter $A$ is the parameter which controls non-stationarity. 
Ideally, optimizing our lower-bound, we should see $A$ converging to 0.
The environment is not relevant for this study since we optimize the hyper-policy solely on the variance term, discarding $\overline{J}_{T,\alpha,\beta}(\vrho)$. 
However, for completeness, we report that the environment is a contextual bandit, where the context follows a sinusoidal function. 

From \cref{fig:bandit_A} we see that the most efficient methods when it comes to making the hyper-policy stationary, and thus push $A$ toward 0, are \textit{uniform $\Psi$}, \textit{two steps $\Psi$ first}, \textit{two steps $\Phi$ first} and the \textit{direct optimization with reset}.

From \cref{fig:bandit_var} we see however that the upper-bound is less smooth for the convex optimization based approaches. Moreover, the upper-bound is tighter for \textit{uniform $\Psi$} and \textit{two steps $\Psi$ first}. 
The final choice of our bound in practice will thus be made between those two. 
The bounds have similarities but we believe that the \textit{two steps $\Psi$ first} may better adapt to other scenarios, since setting one set of the variational parameters to a uniform distribution as in \textit{uniform $\Psi$} doesn't seem to be a robust choice.


\begin{figure*}[t]
\begin{minipage}{.47\textwidth}
\centering
  \includegraphics[scale=0.36]{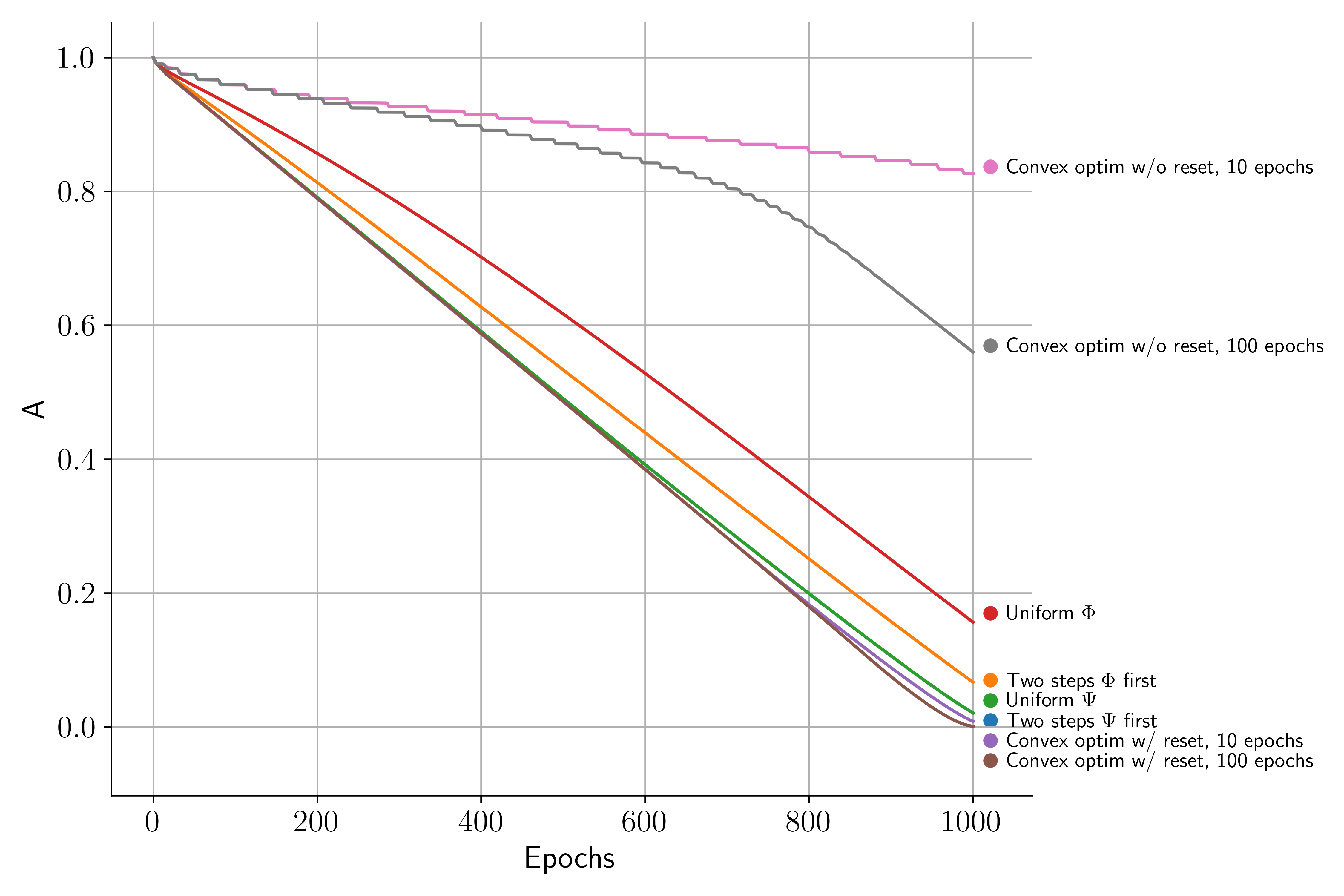}
  \captionof{figure}{Evolution of the scale parameter for several approaches on the upper-bounds of the variance. Note that the value of A for the \textit{two steps $\Psi$ first} and \textit{uniform $\Psi$} are confounded.}
  \label{fig:bandit_A}
\end{minipage}
\hfill
\begin{minipage}{.47\textwidth}
\centering
  \includegraphics[scale=0.36]{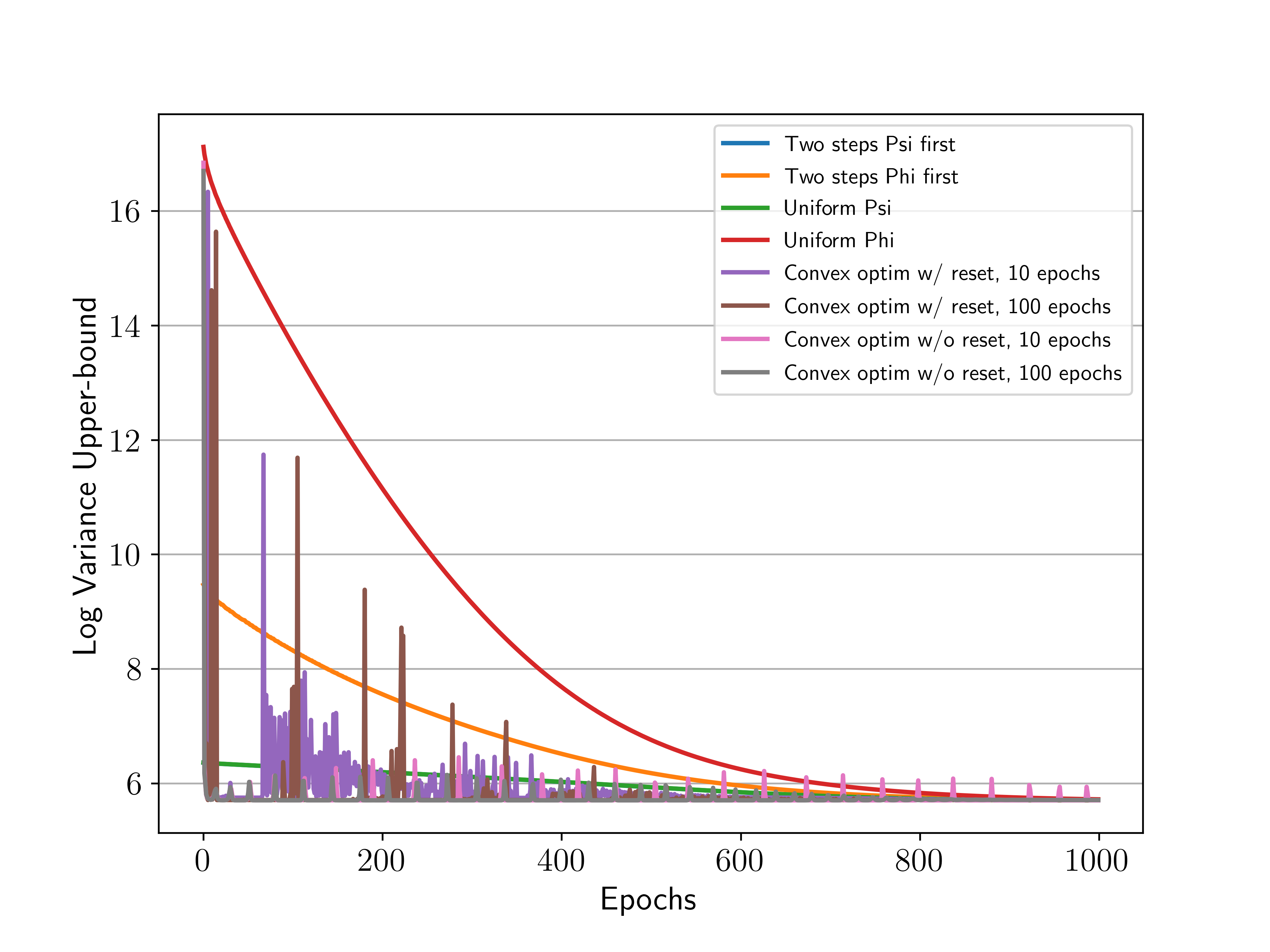}
  \captionof{figure}{Evolution of the variational upper-bound on the variance for several approaches. Here again, the log upper-bound for the \textit{two steps $\Psi$ first} and \textit{uniform $\Psi$} are confounded.}
  \label{fig:bandit_var}
\end{minipage}
\end{figure*}

\section{Experiments Details}
\label{app:exp_details}

All experiments are conducted on Python using Pytorch \cite{NEURIPS2019_9015} as deep learning library. 
Below, we give more details the parameters of each algorithm used in our experiments.

\subsection{Hyper-parameters}
\subsubsection{POLIS}
The hyper-policy has 1046 parameters for the Trading environment and 1040 for the Dam. This is due to the different size of the environment state in the Dam and Trading environments.
Indeed, the state has two entries for the Trading and one for the Dam. This makes 3 parameters for the affine policy for Trading and 2 for the Dam.

Recall that we have set $\alpha=500$, $\gamma=\omega=1$. For the gradient optimization, we use RMSprop \cite{Tieleman2012} with a learning rate of $1e-3$, smoothing constant $\alpha_{RMSprop}$ of 0.9 and parameter $\epsilon_{RMSprop}=1e-10$ for numerical stability. Before each gradient step, we sample 100 replayed trajectory on the last $\alpha$ steps in order to take into account the hyper-policy's stochasticity.
We randomly initialize the hyper-policy's convolutional parameters using Pytorch's default initialization. 
We set the size of the channels inside the temporal convolution of the hyper-policy to $[8,8,4]$, their kernel size to 3 and the positional encoding's dimension to 8. 
We use the same hyper-policy in the behavioural and target period  at the exception of the vector of policy parameters' standard deviation $\boldsymbol{\sigma}$. For the behavioural hyper-policy it is set for all entries to $e^{\frac12}$. For the target hyper-policy, whether it will be fixed during training or not, its initial value is set to a vector of $e^{-1}$.

The remaining parameters are the one which will vary during hyper-parameter search. They are the variance regularization level $\lambda$, the steps for estimation of future performance $\beta$ and whether or not to fix $\boldsymbol{\sigma}$ during training. 
We tested values from $\{10,100,1000\}$ for $\lambda$ $\{10,100,50\}$ for $\beta$ and whether or not to fix $\boldsymbol{\sigma}$ on the Trading environment. For the Dam, we also tested whether or not to fix $\boldsymbol{\sigma}$. We explored values of $\lambda$ in $\{0,10,100\}$ and $\beta$ in $\{10,50,100\}$.

We report the best performing set of hyper-parameters in \cref{tab:opt_params_polis}.

\begin{table}[]
\centering
\begin{tabular}{l|rrrr}
  & \multicolumn{1}{c}{Dam} & \multicolumn{1}{c}{Trading (1)} & \multicolumn{1}{c}{Trading (2)} &  \\ \hline\hline 
$\lambda$ &  100 & 10 & 10 &   \\
$\beta$ &  50 & 500 & 100 &     \\
Fix $\boldsymbol{\sigma}$ &  False & True & True &   
\end{tabular}
\caption{Optimal parameters for the POLIS algorithm on each dataset. The Trading (1) has hyper-parameters selected on the return of 2009-2012 dataset while Trading (2)'s hyper-parameters are selected on the mean return of 2013-2016 and 2017-2020.}
\label{tab:opt_params_polis}
\end{table}

\subsubsection{Stationary Hyper-policy}
The stationary hyper-policy shares most of its parameters with POLIS. Of course, it doesn't use the temporal convolution and the positional encoding and its optimization involves only the $\alpha$-step behind expected return, so $\beta$ and $\lambda$ are not used. The standard deviation $\boldsymbol{\sigma}$ can still be learned or fixed, we report the best performing set of hyper-parameters in \cref{tab:opt_params_polis_stat}.

\begin{table}[h]
\centering
\begin{tabular}{l|rrrr}
  & \multicolumn{1}{c}{Dam} & \multicolumn{1}{c}{Trading (1)} & \multicolumn{1}{c}{Trading (2)} &  \\ \hline\hline
Fix $\boldsymbol{\sigma}$ &  False & True & True &   
\end{tabular}
\caption{Optimal parameters for the Stationary hyper-policy on each dataset. The Trading (1) has hyper-parameters selected on the return of 2009-2012 dataset while Trading (2)'s hyper-parameters are selected on the mean return of 2013-2016 and 2017-2020.}
\label{tab:opt_params_polis_stat}
\end{table}

\subsubsection{Baselines}
Although the three baselines Pro-OLS, Pro-WLS and ONPG implement different ideas, their parameters are similar. 
They all proceed to a number of inner optimizations steps each time the policy is updated, which we set to 10. We also set the standard deviation of the normal distribution under which the action is selected to $0.5$. 
The policy learns features from the state with a neural network which we refer to as state representation module. We refer to the policy module as the neural network which takes  the output features of the state representation and outputs the actions. We set the learning rate of both the state representation module and the policy module to $1e-2$. The policy module uses a hidden layer of 16 neurons. The number of neurons of the state representation module is discussed below. 
We use a buffer size of 1000 and a maximum horizon inside the buffer of 150.

We then proceed to a grid search over following the recommendations for the hyper-parameters values of \cite{chandak2020optimizing}. For the entropy regularization parameter $\lambda_{entropy}$, we consider the set $\{0,1e-3,1e-2\}$. for the importance clipping threshold $t_{IS}$, $\{10,15\}$. The state representation module's neurons per layer $n_{neurons}$ is explored in $\{32,64,[32, 32]\}$, where $[32,32]$ corresponds to 2 layers of 32 neurons. We consider the set $\{5,7\}$ for the size of the extrapolator Fourier basis $k_{Fourier}$ which is used in the performance prediction. We consider the set $\{1,3,5\}$ for the number of step ahead to predict the performance $\delta$. 

\begin{table}[h]
\centering
\begin{tabular}{l|rrr|rrr|rrr}
  & \multicolumn{3}{c|}{Dam}  & \multicolumn{3}{c|}{Trading (1)}  & \multicolumn{3}{c}{Trading (2)}  \\ \hline
  & \multicolumn{1}{c}{Pro-OLS}  & \multicolumn{1}{c}{Pro-WLS} & \multicolumn{1}{c|}{ONPG} &  \multicolumn{1}{c}{Pro-OLS} & \multicolumn{1}{c}{Pro-WLS} & \multicolumn{1}{c|}{ONPG}  & \multicolumn{1}{c}{Pro-OLS} & \multicolumn{1}{c}{Pro-WLS} & \multicolumn{1}{c}{ONPG} \\ \hline\hline
$\lambda_{entropy}$ &0&$1e-2$&$1e-2$ &$1e-3$&$1e-2$&0 &0&$1e-2$&$1e-3$   \\
$t_{IS}$ &10&15&15 &10&15&10 &15&15&15    \\
$n_{neurons}$ &[32,32]&[32,32]&[32,32] &64&32&32 &[32,32]&[32,32]&64  \\
$k_{Fourier}$ &7&7&7 &7&7&5 &7&7&5 \\
$\delta$ &1&5&5 &5&1&1 &1&5&5
\end{tabular}
\caption{Optimal parameters for the Pro-OLS, Pro-WLS and ONPG baselines. The Trading (1) has hyper-parameters selected on the return of 2009-2012 dataset while Trading (2)'s hyper-parameters are selected on the mean return of 2013-2016 and 2017-2020.}
\label{tab:opt_params_baselines}
\end{table}

\subsection{Description of Dataset}

\subsubsection{Dam}
Recall that the cost that the agent gets is a convex combination of the cost related to flooding and the one for not meeting the daily demand. The parameters of this convex combination are, respectively, 0.3 and 0.7 for the first inflow profile, 0.8 and 0.2 for the second and 0.35 and 0.65 for the last one.
We give the mean inflow throughout the year for each profile in \cref{fig:dam_inflows}.

\label{app:dam_dataset}
The three inflows used for the Dam environment are given in 
\begin{figure*}[t]
  \centering
  \includegraphics[scale=0.36]{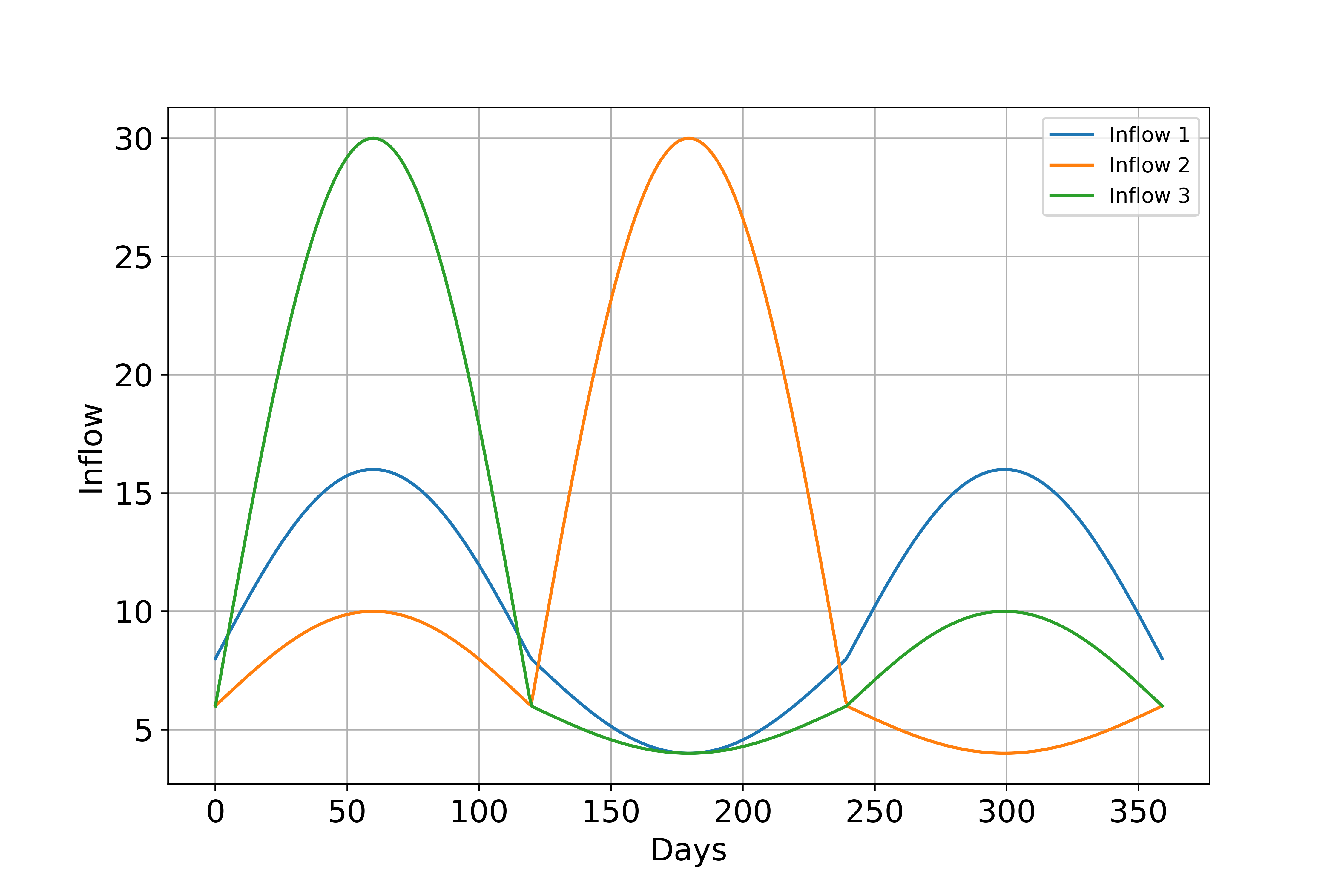}
  \caption{The three inflow profiles for the Dam experiment.}
  \label{fig:dam_inflows}
\end{figure*}

\subsubsection{Trading}
The Trading dataset is composed of the day price of the EUR-USD rate as given in \cref{fig:eurusd_dataset}.

\begin{figure*}[t]
\centering

\begin{subfigure}{.3\textwidth}
  \centering
  \includegraphics[width=\linewidth]{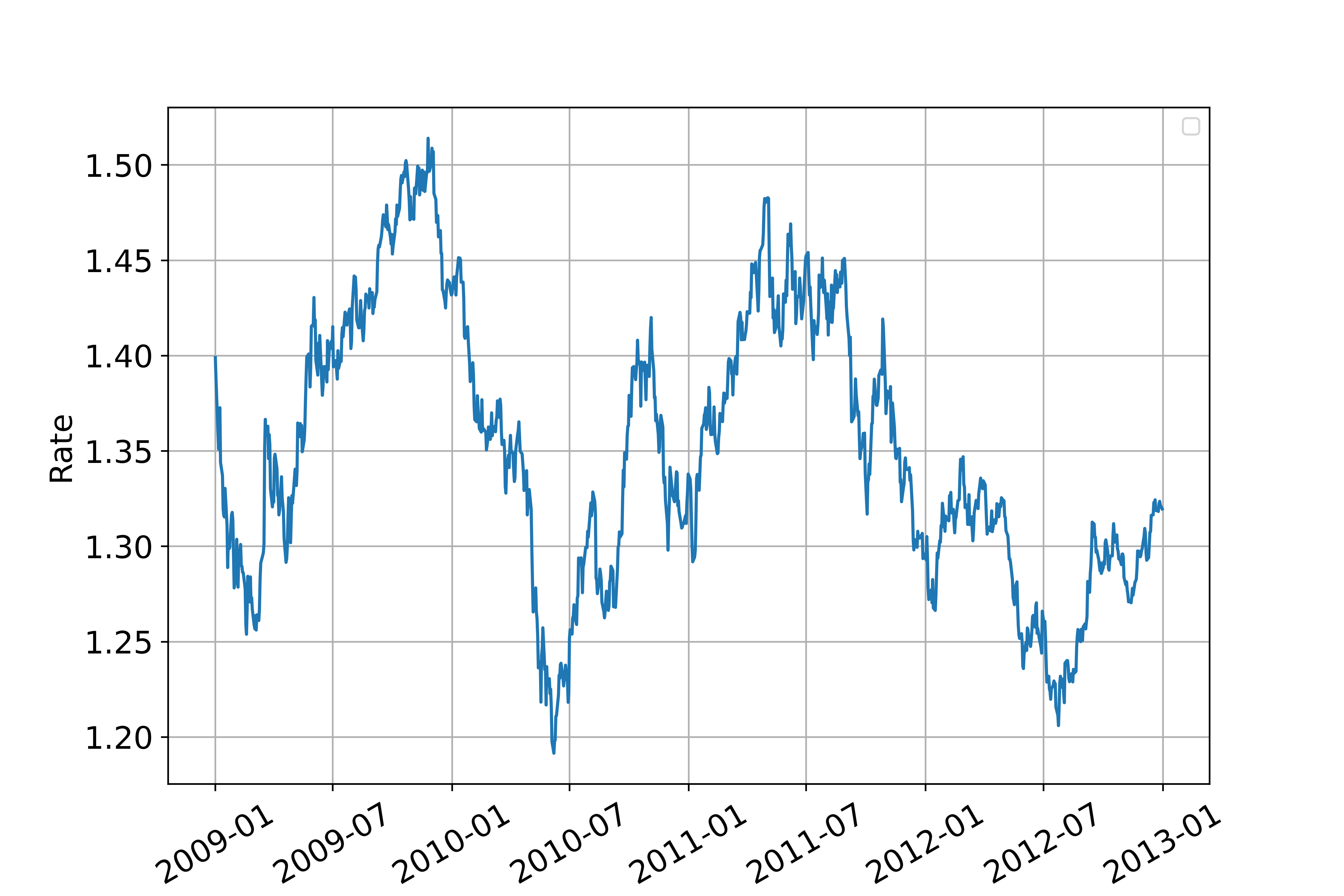}
          \caption*{Dataset 2009-2012.}
\end{subfigure}%
\hfill
\begin{subfigure}{.3\textwidth}
  \centering
  \includegraphics[width=\linewidth]{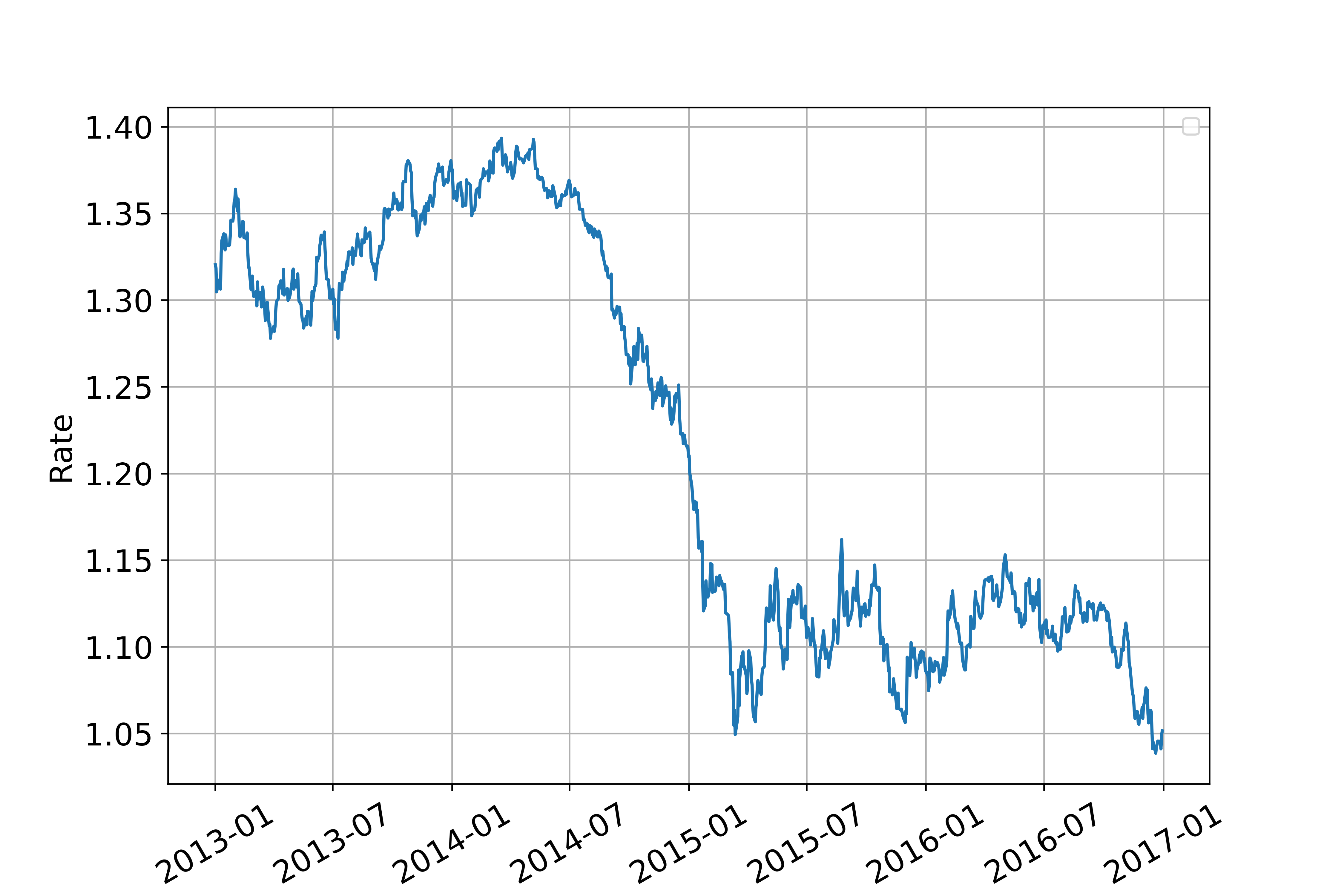}
          \caption*{Dataset 2013-2016.}
\end{subfigure}%
\hfill
\begin{subfigure}{.3\textwidth}
  \centering
  \includegraphics[width=\linewidth]{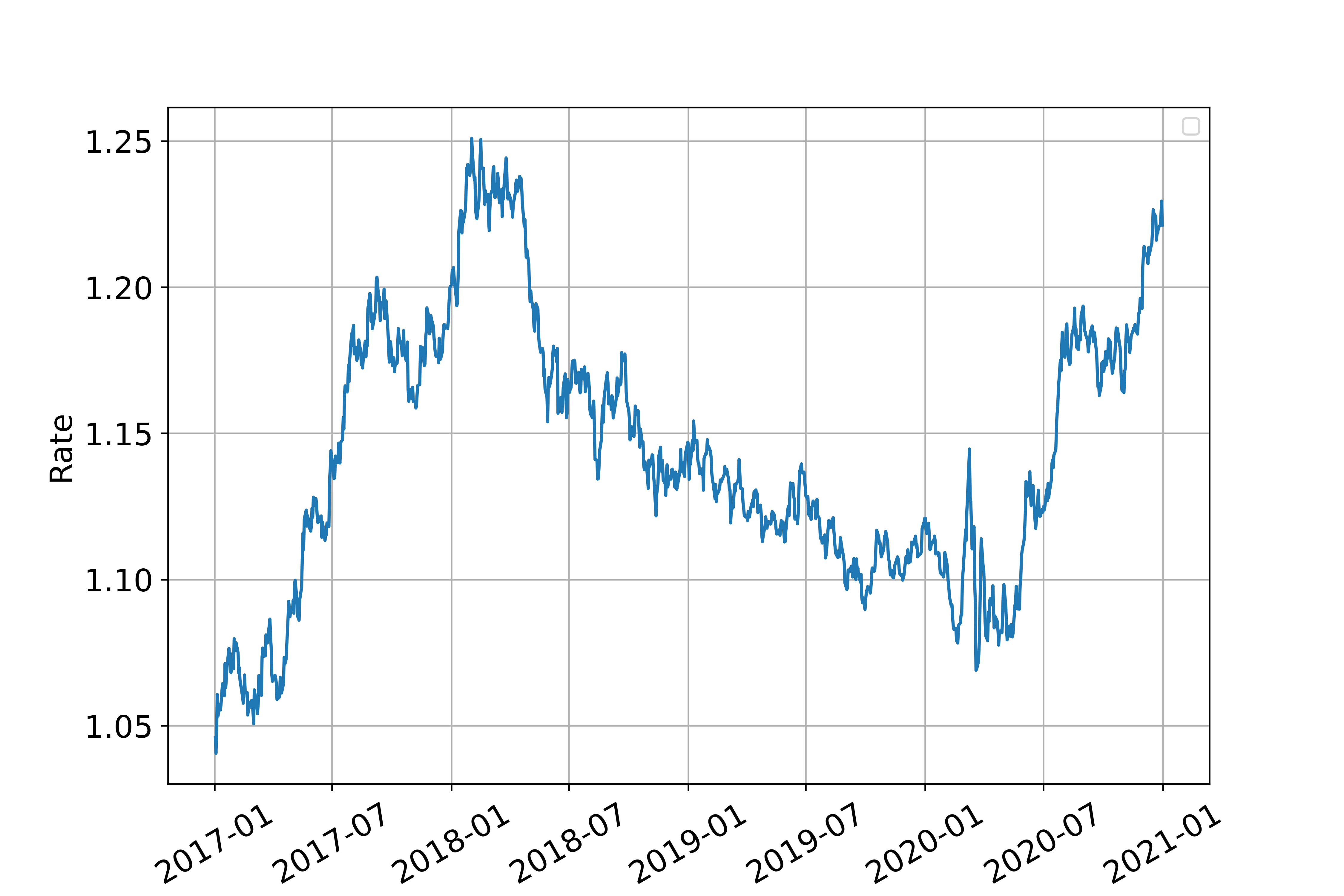}
          \caption*{Dataset 2017-2020.}
\end{subfigure}%
\captionof{figure}{Value of the rate of the EUR-USD on the period 2009-2020, divided in the three datasets used in experiments.}\label{fig:eurusd_dataset}
\end{figure*}

\subsection{Further Experiments}

We report in this section extra experiments for the POLIS algorithm on the Vasicek process. We study values of $\lambda$ in the range $[1,100]$ and values of $\beta$ in the range $(1,100]$ ($\beta=1$ is not considered as it does not involve a mixture of distribution in the variance bound). 
We are interested in the trade-off between the return and the standard deviation of the rewards. 
The results are reported in \cref{fig:pareto_vasicek} and \cref{fig:heatmap_vasicek}. The result suggest that a smaller $\beta$ allows for a general better return a the expense of a higher standard deviation. 
The dependence on $\lambda$ is less clear. However it can be seen that generally, smaller values of $\lambda$ have the same effect as the smaller values of $\beta$.
For a clearer view of this experiment, we report the same plots but for only one value of the other paramter in \cref{fig:pareto_vasicek_one}.

\begin{figure*}[t]
\centering

\begin{subfigure}{.45\textwidth}
  \centering
  \includegraphics[width=\linewidth]{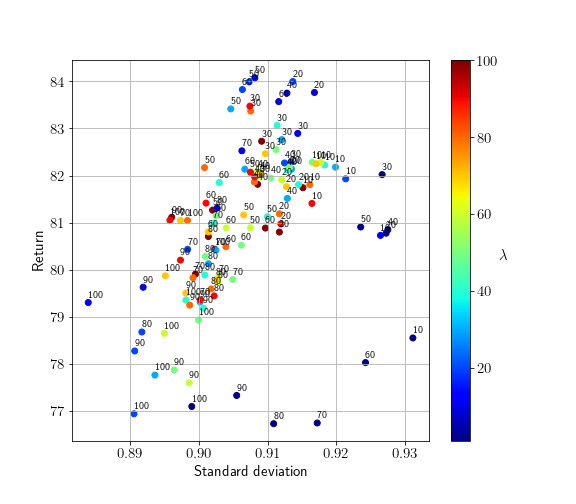}
\end{subfigure}%
\hfill
\begin{subfigure}{.45\textwidth}
  \centering
  \includegraphics[width=\linewidth]{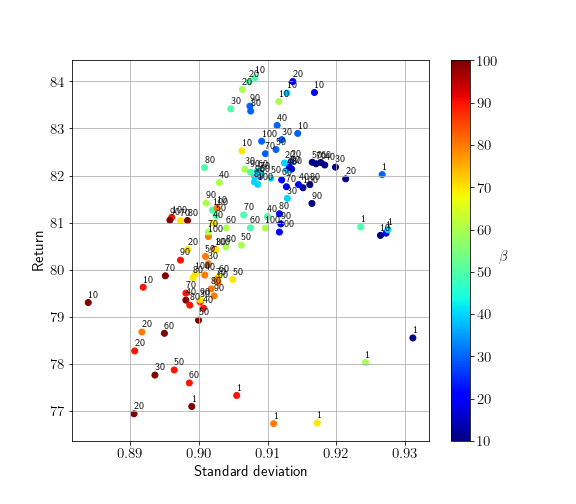}
\end{subfigure}%
\hfill
\captionof{figure}{Plot of the standard deviation of the rewards and return of POLIS on the Vasicek process for different values of $\lambda$ and $\beta$. The value of one parameter is shown on a color scale while the other parameter is labelled on each point. Left figure uses a color scale for $\lambda$ and right for $\beta$. }\label{fig:pareto_vasicek}
\end{figure*}

\begin{figure*}[t]
\centering

\begin{subfigure}{.45\textwidth}
  \centering
  \includegraphics[width=\linewidth]{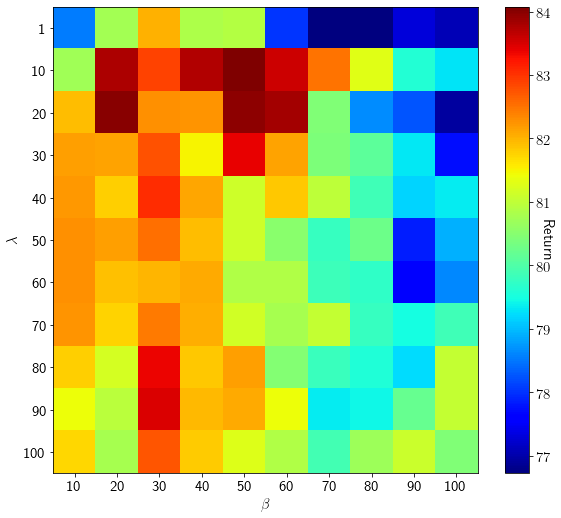}
\end{subfigure}%
\hfill
\begin{subfigure}{.45\textwidth}
  \centering
  \includegraphics[width=\linewidth]{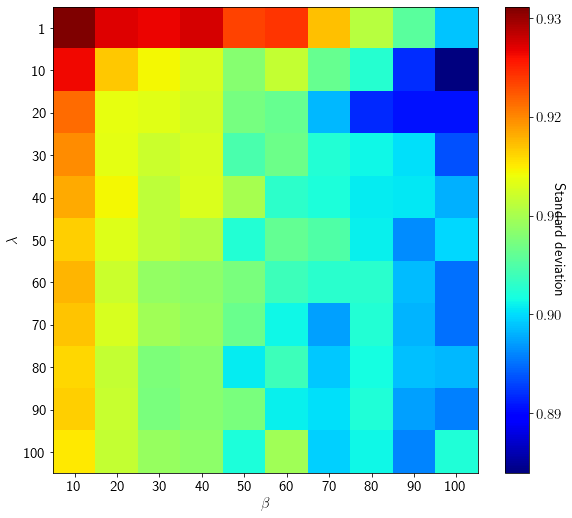}
\end{subfigure}%
\hfill
\captionof{figure}{Return (left) and standard deviation of the rewards (right) of POLIS on the Vasicek process for different values of $\lambda$ and $\beta$. }\label{fig:heatmap_vasicek}
\end{figure*}

\begin{figure*}[t]
\centering

\begin{subfigure}{.45\textwidth}
  \centering
  \includegraphics[width=\linewidth]{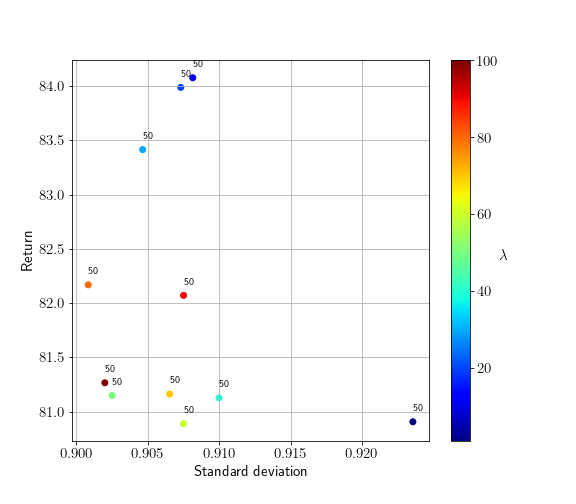}
\end{subfigure}%
\hfill
\begin{subfigure}{.45\textwidth}
  \centering
  \includegraphics[width=\linewidth]{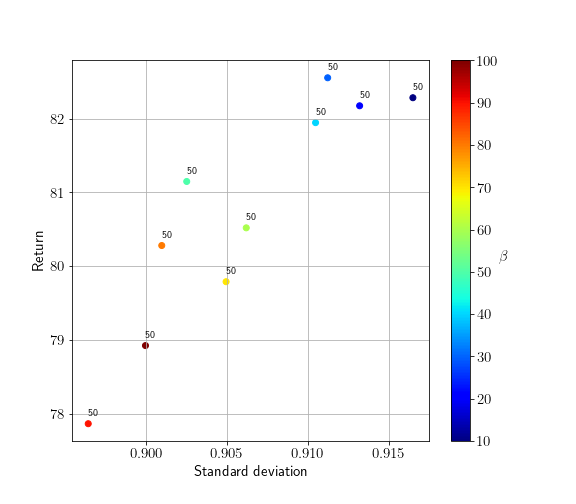}
\end{subfigure}%
\hfill
\captionof{figure}{Plot of the standard deviation of the rewards and return of POLIS on the Vasicek process for different values of $\lambda$ and $\beta$. The value of one parameter is shown on a color scale while the other parameter is labelled on each point. Left figure uses a color scale for $\lambda$ and $\beta$ is fixed to 50 while right figure uses a color scale for $\beta$ and  $\lambda$ is fixed to 50. }\label{fig:pareto_vasicek_one}
\end{figure*}



\end{document}